\def\year{2022}\relax
\documentclass[letterpaper]{article} 
\usepackage{aaai22}  
\usepackage{times}  
\usepackage{helvet}  
\usepackage{courier}  
\usepackage[hyphens]{url}  
\usepackage{graphicx} 
\usepackage{natbib}  
\usepackage{caption} 
\DeclareCaptionStyle{ruled}{labelfont=normalfont,labelsep=colon,strut=off} 
\usepackage{algorithm}
\usepackage{algpseudocode}
\usepackage{amsmath,amssymb,amsfonts}
\usepackage{amsthm}
\usepackage{graphicx}
\usepackage[caption=false,font=footnotesize]{subfig}

\usepackage[short]{optidef}
\usepackage[dvipsnames]{xcolor}

\newcommand*\yz[1]{{\color{black}#1}}
\newcommand*\yzeq{\color{black}}

\newcommand{\bE}{\mathbb{E}}
\newcommand{\bP}{\mathbb{P}}
\newcommand{\bK}{\mathbb{K}}
\newcommand{\I}{\mathbb{I}}

\newcommand{\piLy}{{\pi_{\tt LyOn}}}
\newcommand{\piOff}{{\pi_{\tt LyOff}}}
\newcommand{\piOpt}{{\pi_{\tt Opt}}}

\newtheorem{assumption}{Assumption}

\newtheorem{definition}{Definition}
\newtheorem{proposition}{Proposition}

\newtheorem{remark}{Remark}
\newtheorem{lemma}{Lemma}
\newtheorem{theorem}{Theorem}

\usepackage{newfloat}
\usepackage{listings}
\floatstyle{ruled}
\newfloat{listing}{tb}{lst}{}
\floatname{listing}{Listing}
\title{A Lyapunov-Based Methodology for Constrained \\ Optimization with Bandit Feedback}
\author{
    Semih Cayci,\textsuperscript{\rm 1}\equalcontrib
    Yilin Zheng,\textsuperscript{\rm 2}\equalcontrib
    Atilla Eryilmaz\textsuperscript{\rm 2}
}
\affiliations{
    \textsuperscript{\rm 1}Coordinated Science Laboratory, University of Illinois at Urbana-Champaign, Urbana, IL 61801.  \\
    \textsuperscript{\rm 2} Department of Electrical and Computer Engineering, The Ohio State University, Columbus, OH 43210\\
   
    scayci@illinois.edu, zheng.1443@osu.edu, eryilmaz.2@osu.edu
}

\usepackage{bibentry}

\begin{document}

\maketitle

\begin{abstract}
In a wide variety of applications including online advertising, contractual hiring, and wireless scheduling, the controller is constrained by a stringent budget constraint on the available resources, which are consumed in a random amount by each action, and a stochastic feasibility constraint that may impose important operational limitations on decision-making. In this work, we consider a general model to address such problems, where each action returns a random reward, cost, and penalty from an unknown joint distribution, and the decision-maker aims to maximize the total reward under a budget constraint $B$ on the total cost and a stochastic constraint on the time-average penalty. We propose a novel low-complexity algorithm based on Lyapunov optimization methodology, named ${\tt LyOn}$, and prove that for $K$ arms it achieves $O(\sqrt{\yz{K}B\log B})$ regret and zero constraint-violation when $B$ is sufficiently large. The low computational cost and sharp performance bounds of ${\tt LyOn}$ suggest that Lyapunov-based algorithm design methodology can be effective in solving constrained bandit optimization problems.
\end{abstract}

\section{Introduction}
Multi-armed bandits (MAB) have been predominantly used to model exploration-and-exploitation problems since its inception \cite{robbins1952some, lai1985asymptotically, berry1985bandit}. As a consequence of the universality of the dilemma, bandit algorithms have found a broad range of applications from medical trials and adaptive routing to server allocation \cite{bubeck2012regret}. In many applications of interest, the controller is required to satisfy multiple constraints while achieving the optimal expected total reward under a given finite budget. 

To take one example, in fair resource allocation problems, such as task scheduling or contractual hiring, each arm (e.g., a user or a social group) must receive at least a given fraction of the total budget (e.g., total time) while maximizing the total reward. Other examples from diverse domains, such as wireless resource allocation, online advertising, etc., also take this form (see Section~\ref{sec:problem} for more discussion). In order to solve fundamental learning applications such as these, a fast and effective constrained bandit optimization framework is required.

This motivates us in this paper to formulate and solve the following budget-constrained bandit problem in a stochastic setting. Each arm pull takes a random and arm-dependent resource (e.g. time, energy, etc.) from the budget, and the decision-making process continues until the total consumed resource exceeds a given budget $B > 0$. At the end of each arm pull, the controller receives a random reward and a random penalty. The objective of the controller is to maximize the expected total reward subject to an inequality constraint on the average penalty per unit resource consumption.

\subsection{Main  Contributions}
In this work, we tackle the aforementioned general constrained optimization problem with bandit feedback, and propose a novel Lyapunov-based design methodology to develop efficient algorithms that achieve sharp convergence results. Our main contributions can be summarized as follows:
\begin{itemize}
    \item \textit{General model:} We consider a generic constrained bandit optimization problem (in Section~\ref{sec:problem}), which: (i) incorporates random costs for each action; (ii) is subject to a stringent budget (knapsack) constraint; and (iii) has stochastic feasibility constraints as required by many applications.
    \item \textit{Lyapunov methodology for bandit optimization:} Based on a Lyapunov-drift minimization technique from stochastic control,  we design novel low-complexity 
    bandit algorithms with provably sharp convergence properties. This approach suggests a general design methodology (outlined in Section~\ref{sec:outline}) that can be utilized in other constrained bandit optimization scenarios.
    \item \textit{Analysis techniques:} We also employ new analysis techniques (in Section~\ref{sec:LyaOpt}) for the reward maximization problem subject to stochastic and knapsack constraints based on a combination of renewal theory, stochastic control, and bandit optimization.
\end{itemize}

\subsection{Related Work}
\label{sec:related}

Knapsack-constrained bandit problem was considered in \cite{badanidiyuru2013bandits}, where the objective of the controller is to maximize the expected total reward under a stringent budget constraint. The authors proposed a learning algorithm with $O(\sqrt{B})$ problem-independent regret. Bandit algorithms with $O(\log(B))$ problem-dependent regret bounds were proposed under various extensions of the knapsack-constrained bandit models \cite{tran2012knapsack, flajolet2015logarithmic, xia2015thompson, xia2016budgeted, cayci2019learning, cayci2020budget}. These works differ from ours in that, while the controller is constrained by stringent knapsack constraints, stochastic feasibility constraints are not accommodated.

As an extension of \cite{badanidiyuru2013bandits}, finite-armed bandit problem with stochastic feasibility constraints was considered in \cite{agrawal2014bandits}, and a UCB-type algorithm with $O(\sqrt{B})$ regret and $O(1/\sqrt{B})$ constraint-violation was proposed. This setting was extended to episodic Markov decision processes in \cite{qiu2020upper}, and similar order results for regret and constraint-violation were obtained. In these models, each action incurs a unit cost. Furthermore, the proposed algorithms in \cite{agrawal2014bandits} require solving a convex optimization / linear programming problem at each stage. Our model accommodates random cost, which is subject to a knapsack constraint, in addition to the stochastic feasibility constraint. Based on Lyapunov optimization theory, we develop computationally-efficient iterative algorithms .

Lyapunov optimization methods have been widely used in stochastic network optimization and queueing systems (see \cite{neely2010dynamic, neely2010stochastic, georgiadis2006resource} and references therein). This methodology was later used in convex optimization problems \cite{yu2016low} with known gradients. In these approaches, a predominant assumption is that the random system state is known to the controller prior to its decision, therefore the existing methods do not work for online learning setting where the controller does not have the system state or the system statistics before making a decision. 

Lyapunov optimization methods were first used in the context of online learning in \cite{cayci2020group}, where the goal of the controller is to maximize the total utility as a function of each arm's time-average reward subject to a knapsack constraint under delayed experts feedback. Our work extends \cite{cayci2020group} in that we consider bandit feedback and also incorporate stochastic feasibility constraints in this paper. \yz{Some recent works \cite{liu2020pond,liu2021efficient}, which utilized Lyapunov-drift methods for online learning, studied the online-dispatching and linear bandits with cumulative constraints. Our work substantially differs from these works in that we incorporate knapsack budget constraints and random costs per arm selection.} 

\section{Constrained 
Reward Maximization Problem with Bandit Feedback} \label{sec:problem}

We consider a finite-armed bandit problem with $K>1$ arms, and the set of arms denoted by $\bK=\{1,2,\ldots, K\}$. If arm $k$ is chosen at $n^{th}$ epoch, it incurs a cost of $X_{n,k}$, yields a reward of $R_{n,k}$, and returns a penalty of $Y_{n,k}$, where the outcome of the joint random vector $(X_{n,k}, R_{n,k}, {Y}_{n,k})$ is learned via bandit feedback at the end of each arm decision. We assume that the random process $\{(X_{n,k},R_{n,k},Y_{n,k}):n\geq 1\}$ is independent and identically distributed over $n$, and independent across different arms for all $k\in\mathbb{K}$. For simplicity, we assume that $X_{n,k},R_{n,k},Y_{n,k}\in[0,1]$ for all $n,k$, which can be easily extended to general sub-Gaussian random variables by using the same techniques used in this paper. The controller has a total budget $B>0$ at the beginning of the process, and tries to maximize the expected cumulative reward under time-average constraints on the penalties by sampling the arms wisely under this budget constraint. 

\yz{Note that stochastic constraints and budget constraints imply completely different system dynamics. Violation of a budget constraint immediately stops the decision-making process. On the other hand, the stochastic constraints are aimed to be satisfied asymptotically, while instantaneous violations do not stop the decision-making process.}

First, we introduce the \yz{causal} policy space.

\begin{definition}[Causal Policy]
Let $\pi$ be a policy that yields a sequence of arm pulls $\{I_n^\pi\in \bK:n\geq 1\}$. Under $\pi$, the history until epoch $n$ is the following filtration:
\begin{equation}
    \mathcal{F}_n^\pi=\sigma(\{(I_j^\pi, X_{j,k}, R_{j,k}, {Y}_{j,k}):I_j^\pi = k, 1\leq j \leq n\}),
\end{equation}
\noindent where $\sigma(Z)$ denotes the sigma-field of a random variable $Z$. We call an algorithm $\pi$ causal if $\pi$ is non-anticipating, i.e., $\{I_n^\pi = k\}\in\mathcal{F}_{n-1}^\pi$ for all $k, n$.
\end{definition}

The set of all causal policies is denoted as $\Pi$. We denote the variables at epoch $n$ under policy $\pi$ as $X_n^\pi = X_{n,I_n^\pi}$, $R_n^\pi = R_{n,I_n^\pi}$ and $Y_{n}^\pi = Y_{n,I_n^\pi}$. The total cost incurred in $n$ epochs under an causal policy $\pi\in\Pi$ is a controlled random walk which is defined as $S_n^\pi = \sum_{i=1}^nX_i^\pi.$ The decision process under a policy $\pi$ continues until the budget $B$ is depleted. We assume that the reward corresponding to the final epoch during which the budget is depleted is gathered by the controller. Thus, the total number of pulls under $\pi$ is a random variable that is defined as follows:
\begin{equation}
    N^\pi(B) = \inf\Big\{n \geq 1: S_n^\pi > B\Big\}.
\end{equation}

Note that the total number of pulls $N^\pi(B)$ is a stopping time adapted to the filtration $\{(\mathcal{F}_n^\pi): n \geq 0\}$. Accordingly, the cumulative reward under a policy $\pi$ can be written as follows:
\begin{equation}
    \label{eqn:reward}
    {\tt REW}^\pi(B) = \sum_{n=1}^{N^\pi(B)}R_n^\pi.
\end{equation}
Then, we can write the generic problem formulation considered in this paper as follows:
\begin{align}
\sup_{\pi\in\Pi}& \quad \bE[{\tt REW}^\pi(B)], \notag\\ 
\textrm{subject to:} &\quad  \mathbb{E}\Bigg[\frac{1}{B}\sum_{n=1}^{N^\pi(B)}Y_{n}^\pi\Bigg]\leq c.
\label{eqn:optimization-problem}
\end{align}

\begin{definition}[(Pseudo) Regret and constraint-violation]
Let $\piOpt$ be the solution of \eqref{eqn:optimization-problem} and ${\tt OPT}(B)=\bE[{\tt REW}^{\piOpt}(B)]$. For any causal policy $\pi\in\Pi$ and a budget $B>0$ level, the (pseudo) regret, ${\tt REG}^\pi(B),$ and constraint-violation, $D^\pi(B)$, are defined as follows:
\begin{align}
    {\tt REG}^\pi(B) = {\tt OPT}(B)-\bE[{\tt REW}^\pi(B)], \\
    D^\pi(B) = \mathbb{E}\Bigg[\frac{1}{B}\sum_{n=1}^{N^\pi(B)}Y_{n}^\pi\Bigg]-c.
\end{align}
\end{definition}

The objective of this paper is to design low-complexity bandit algorithms that are guaranteed to give a low regret and a vanishing constraint-violation level that decays rapidly with the budget level $B$, in the absence of any statistical knowledge on the costs, rewards, and penalties. The generic problem \eqref{eqn:optimization-problem} 
has numerous applications in communications, control, operations research and management, whereby optimal decision-making under data scarcity and uncertainty is common. Next, we provide a detailed application in wireless scheduling in next generation networks, and refer the reader to Appendix~\ref{app:applications} for other example applications for contractual hiring, and online advertising. 

\paragraph{{Application to Next Generation} Wireless Scheduling with Quality-of-Service Guarantees:} 

Next-generation wireless technologies are required to serve a highly dynamic population of users with stringent quality-of-service (QoS) guarantees (such as low delay \cite{khalek2014delay}) over ultra-high frequency bands with nontraditional statistical and temporal characteristics \cite{rappaport2015millimeter}.
As such, existing estimation and allocation techniques that rely strongly on persistent users and slowly-changing nature and known statistical models of channel conditions are no longer suitable for use in this new ultra-wideband communication paradigm. The controller is required to learn how to optimize the throughput subject to QoS guarantees by using the ARQ (bandit) feedback received after each transmission. 

This calls for the design of time/energy-constrained point-to-point communication solutions over $K$ parallel memoryless channels with unknown and diverse statistical characteristics. In particular, each connection starts with a total time or energy budget of $B$ units. The transmission of $n^{th}$ packet over the $k^{th}$ channel consumes a random amount of $X_{n,k}$ resource (e.g., transmission time or energy), yields a reward (e.g., throughput) $R_{n,k}$ and incurs a penalty $Y_{n,k}$ upon completion of transmission. In this context, $Y_{n,k}$ is a generic penalty that will be used in modeling time-average quality-of-service guarantees. As an example, consider delay-constrained communication, where the arriving packets should be transmitted in a timely manner. Then, for a given deadline level $d\in[0,1]$, we let $Y_{n,k} = \I\{X_{n,k} > d\}$, which counts the number of packets that are delayed for more than $d$ time units. For a given time or energy budget $B$ and a quality-of-service constraint $c$, the optimization problem \eqref{eqn:optimization-problem} leads to throughput maximization subject to a guarantee on the time-average number of delayed packets. Note that many other QoS criteria, such as the fraction of dropped packets, can be modeled in a similar manner, which implies the generality of this approach.

\section{Outline of the Lyapunov-Based Design Methodology and Main Results}
\label{sec:outline}

In this work, we develop a low-complexity online algorithm for solving the generic constrained reward maximization problem (\ref{eqn:optimization-problem}) by employing a Lyapunov-drift minimization methodology. Since this methodology may be of independent value, in this section we provide an outline of its main steps along with an informal discussion of the key results we obtained under them. 

\paragraph{(i) Characterization of the Asymptotically-Optimal Stationary Randomized Oracle:}  
The optimization problem described in Section \ref{sec:problem} is a variant of the unbounded knapsack problem, and it is known that similar stochastic control problems are PSPACE-hard  \cite{badanidiyuru2013bandits, papadimitriou1999complexity}. In Section~\ref{sec:rand} we propose a \textit{stationary randomized} policy {$\pi^*$} in Definition~\ref{def:stat} that achieves (see Proposition~\ref{prop:optimality-gap}) $O(1)$ regret and $O(1/B)$ constraint-violation gap. This proves that the stationary policy is asymptotically optimal as the budget $B$ goes to infinity. 

Our Lyapunov-based policy design, developed in Section~\ref{sec:LyaOpt}, are broken into the following two steps:

\paragraph{(ii) Offline Lyapunov-Drift-Minimizing Policy Design:} We first consider in Section~\ref{sec:LyOff} the `offline' setting with known reward, cost, and penalty statistics. There, we introduce a virtual queue $\{Q_n^\pi\}$ that is updated as:
$ Q_{n+1}^\pi = \max\{0, Q_{n}^\pi + Y_{n}^\pi -(c-\delta)X_n^\pi \}, $
with a design choice $\delta\in[0,c)$, which keeps track of the constraint-violation level under policy $\pi$ over decisions $n\geq 1$. Then, under this queue dynamics, we propose a quadratic Lyapunov drift-minimizing policy $\piOff$ in Definition~\ref{policy:offline} that achieves (cf. Proposition~\ref{prop:offline-opt}) $O((\delta+1/V)B)$ regret and $O(V/B-\delta)$ constraint-violation gap, where $V>0$ is a design parameter. With the particular selection of $V=\Theta(\sqrt{B})$ and $\delta = \Theta(1/\sqrt{B}),$ we can guarantee $O(\sqrt{B})$ regret and \yz{zero} constraint-violation for $\pi_{\tt LyOff}$ \yz{with sufficiently large $B$}.

\paragraph{(iii) Online Lyapunov-Drift-Minimizing Policy Design:} Then, in Section~\ref{sec:LyOn}, we return to the original `online' setting with unknown statistics, and develop a low-complexity empirical Lyapunov-drift minimizing policy $\pi_{\tt LyOn}$ that integrates confidence bounds of proposed empirical estimator with the queueing dynamics from the offline case. Then, the main result of the paper (cf. Theorem~\ref{thm:regret}) establishes that  $\pi_{\tt LyOn}$ achieves  $O(\sqrt{\yz{K}B \log B}+ B (1 + \delta \log B) / V +  B\delta + K \log B)$ regret and $O(K \log B /B + V/B -\delta)$ constraint-violation level. With the particular selection of design parameters as $V=\Theta(\sqrt{B \log B})$ and $\delta = \Theta(\sqrt{\log B/ B}),$ we guarantee $O(\sqrt{\yz{K}B \log B})$ regret and \yz{zero} constraint-violation for $\pi_{\tt LyOn}$ \yz{with sufficiently large $B$.}

The online analysis is especially complicated by the fact that the cumulative reward and penalty processes form  stopped and controlled random walks. To address the associated challenge, we combine techniques from renewal theory and Martingale concentration inequalities \cite{wainwright2019high} to find a high probability upper bound for $N^\pi(B)$. Additionally, for the online policy with unknown statistics, we carefully integrate empirical concentration inequalities \cite{cayci2020budget} with hitting time analysis for Martingales \cite{hajek1982hitting} as well as Lyapunov drift analysis \cite{neely2010dynamic} to bound ${\tt REW}^\pi(B)$ and $D^\pi(B)$. 

\section{Asymptotically-Optimal Stationary Randomized Oracle Design}
\label{sec:rand}

As a tractable benchmark, in this section, we consider approximation algorithms with provably good performance.

\begin{definition}[Reward Rate and Penalty Rate] Consider a stationary randomized policy $\pi=\pi(\mathbf{p})$ for a given probability mass function $\mathbf{p}=(p_1,p_2,\ldots,p_K)$, which takes action $k$ with probability $p_k$ independent from the history. Then, under $\pi(\mathbf{p})$, the reward rate and penalty rate are defined as:
\begin{equation}
    r(\mathbf{p}) = \frac{\sum_{k\in\bK}p_k\bE[R_{1, k}]}{\sum_{k\in\bK}p_k\bE[X_{1,k}]},\quad 
    y(\mathbf{p}) = \frac{\sum_{k\in\bK}p_k\bE[Y_{1, k}]}{\sum_{k\in\bK}p_k\bE[X_{1,k}]},
\end{equation}
\end{definition}

Intuitively, if an arm is chosen persistently according to the stationary randomized policy $\pi(\mathbf{p})$ until the budget $B>0$ is depleted, the cumulative reward becomes $r(\mathbf{p})B+o(B)$ and cumulative penalty becomes $y(\mathbf{p})B+o(B)$.
Moreover, whenever $\bE[R_{1,k}^2]<\infty$ and $\bE[Y_{1,k}^2]<\infty$ (trivially true for bounded random variables), the additive $o(B)$ term is $O(1)$ in both cases by Lorden's inequality \cite{asmussen2008applied}.

In the following, we prove that a stationary randomized policy achieves $O(1)$ optimality gap with constraint-violation vanishing at a rate of $O(1/B)$.

\begin{definition}[Optimal Stationary Randomized Policy, $\pi^*$] \label{def:stat}
Let $\mathbf{p}^*$ be the solution to the following optimization problem:
\begin{equation*}
    \max_{\mathbf{p}\in \Delta_K}\left\{r(\mathbf{p}), \quad \textrm{subject to:} \quad y(\mathbf{p})\leq c\right\}.
\end{equation*}
where $\Delta_K$ is the $K$-dimensional probability simplex.
The \emph{optimal stationary randomized policy}, denoted by $\pi^*$, pulls arm $k$ with probability $p_k^*$ independently at each epoch until the budget is depleted: $\bP(I_n^{\pi^*} = k)=p_k^*$, for all $n\leq N^{\pi^*}(B)$.
\label{def:osrp}
\end{definition}

The main result of this section is the following proposition, which implies that $\pi^*$ is a good approximation algorithm for $\piOpt$ for $B>0$.

\begin{proposition}[Optimality Gap for $\pi^*$]\label{prop:optimality-gap}
For the optimal static policy $\pi^*$ 
for any given $B > 0$, the following regret and constraint-violation gap results hold:
\begin{equation}
{\tt REG}^{\pi^*}(B) = O(1), \quad 
    D^{\pi^*}(B) =O\Big(\frac{1}{B}\Big). 
\end{equation}
Therefore, $\pi^*$ is asymptotically optimal, i.e. $\lim_{B\to \infty} {\tt REG}^{\pi^*}(B)/B = 0$ and $\lim_{B \to \infty} D^{\pi^*}(B) = 0$. 
\end{proposition}
The proof of Proposition \ref{prop:optimality-gap} can be found in Appendix \ref{app:optimality-gap}.
In the next section, we will introduce a learning algorithm to achieve the performance of the optimal stationary randomized policy with low regret and constraint-violation.

\section{Algorithm Design Based on Empirical Lyapunov Drift Minimization}
\label{sec:LyaOpt}

In the previous section, we proved that the stationary randomized policy $\pi^*$ achieves the optimality in offline setting with small optimality gap and constraint-violation, which implies it can be used as a benchmark for the design and analysis of learning algorithms. By using this, we will develop a dynamic learning algorithm based on the Lyapunov-drift-minimization approach. For details about this dynamic optimization approach in offline setting, see \cite{neely2010stochastic, neely2010dynamic}.We refer to Section~\ref{sec:related} for the detailed discussion of the differences from related works in this space.

We make two mild assumptions that are needed for the development and analysis of our design: 
\begin{assumption}[$\epsilon$-Slater Condition]
There exists an arm $k\in\bK$ such that $\bE[Y_{n,k}-cX_{n,k}]\leq -\epsilon$ for some $\epsilon > 0$. We only need $\epsilon$ to be a \yz{positive lower-bound } of the actual value. 
\label{assn:slater}
\end{assumption}
Assumption \ref{assn:slater} is reasonable because for feasibility, either all arms should satisfy $\bE[Y_{n,k}-cX_{n,k}]=0$ for all $k$ or Assumption \ref{assn:slater} should hold, otherwise the constraint can never be satisfied \yz{once it is violated}. Since $\bE[Y_{n,k}-cX_{n,k}]=0,~\forall k \in \bK$ is a trivial case, Assumption \ref{assn:slater} is satisfied in almost all applications.
\begin{assumption} [Bounded Moments]
For all arms $k\in \bK$, assume $ \max_k\frac{\bE[R_{1,k}]}{\bE[X_{1,k}]}\leq r_{\tt max} <\infty$ and  $\max_k\frac{\bE[Y_{1,k}]}{\bE[X_{1,k}]}\leq y_{\tt max} <\infty$. In addition, assume $\sigma^2 = \max_{k\in\bK}\bE\big[\big(Y_{1,k}-cX_{1,k}\big)^2\big] < 1,$ and $\min_k\bE[X_{1,k}] \geq \mu_{\tt min} > 0$. We only need $\mu_{\tt min}$ to be a \yz{lower bound} and $r_{\tt max}$, $y_{\tt max}$ to be \yz{upper bounds} of the actual values.
\label{assn:bounded-moment}
\end{assumption}

This assumption is reasonable because otherwise the optimization problem in \eqref{eqn:optimization-problem} would become either trivial or unsolvable. \yz{For bounded rewards and penalty between $[0,1]$, $r_{\tt max}$ and $y_{\tt max}$ can be upper bounded by $1/\mu_{\tt min}$.}

\subsection{Offline Lyapunov-Drift Minimizing Policy {\tt LyOff} Design}
\label{sec:LyOff}

First, we consider the Lyapunov optimization methods in the offline setting with known first-order statistics by closely following \cite{neely2010dynamic}, while improving the results for finite-time performance by using the drift results in \cite{hajek1982hitting}. As a measure of constraint-violation under a causal policy $\pi\in\Pi$, we define the variables $Q_{n}^\pi$ recursively as follows:
\begin{equation}
    Q_{n+1}^\pi = \max\Big\{0, Q_{n}^\pi + Y_{n}^\pi-(c-\delta)X_n^\pi\Big\},
    \label{eqn:q-length}
\end{equation}
where $Q_{0}^\pi=0$ and $\delta \in [0,c)$ is a fixed parameter that controls the tightness of the constraint. Note that $Q_{n+1}^\pi \in \mathcal{F}_{n}^\pi$ for all $n$ since $\pi$ is causal. Intuitively, the stability of $\{Q_{n}^{\pi}\}_n$ implies that the constraint is satisfied. 
The key metric for decision-making is the Lyapunov drift-plus-penalty ratio, which is defined in the following definition.
\begin{definition}[Lyapunov Drift-plus-Penalty Ratio] \label{def:dppr}
For any given $V > 0$, under a causal policy $\pi$, the Lyapunov drift-plus-penalty ratio is defined as follows:
\begin{align}\label{eqn:lyapunov-dppr}
    \Psi_n(Q_n^\pi) &= -V\frac{\bE[R_{n}^\pi|\mathcal{F}_{n-1}^\pi]}{\bE[X_{n}^\pi|\mathcal{F}_{n-1}^\pi]}+Q_{n}^\pi\frac{\bE[Y_{n}^\pi|\mathcal{F}_{n-1}^\pi]}{\bE[X_{n}^\pi|\mathcal{F}_{n-1}^\pi]}.
\end{align}
\end{definition}
For any stationary randomized policy $\pi(\mathbf{p}_n),$
with $\mathbf{p}_n \in\mathcal{F}_{n-1}$,
the Lyapunov drift-plus-penalty ratio becomes:

\begin{align}\label{eqn:lyapunov-dppr-random}
    \Psi_n(Q_n^{\pi(\mathbf{p}_n)}) &= -V\frac{\sum_{k=1}^K p_{n,k}\bE[R_{n,k}]}{\sum_{k=1}^K p_{n,k}\bE[X_{n,k}]}\notag\\ 
    & \quad +Q_{n}^{\pi(\mathbf{p}_n)}\frac{\sum_{k=1}^K p_{n,k}\bE[Y_{n,k}]}{\sum_{k=1}^K p_{n,k}\bE[X_{n,k}]}.
\end{align}
Intuitively, in the offline setting where all first-order moments are known, a stationary randomized policy $\pi(\mathbf{p}_n)$ that minimizes \eqref{eqn:lyapunov-dppr-random} over all probability distributions in every epoch $n$, achieves a near-optimal trade-off between the cumulative reward and constraint-violation determined by the parameter $V>0$ \cite{neely2010dynamic}. In the following, we outline this result in the offline setting, which will guide us in developing the online algorithm in Section~\ref{sec:LyOn}. 

\begin{definition}[Offline Lyapunov-Drift-Minimizing Distribution] \label{policy:offline} 
For any $n$, let $\mathbf{q}_n^*$ be defined as follows:
\begin{equation}
   \mathbf{q}_n^* \in \arg\min_{\mathbf{p}\in\Delta_K}~\Psi_n(Q_n^{\pi(\mathbf{p})}).
    \label{eqn:drift-min}
\end{equation}
\end{definition}
The problem in \eqref{eqn:drift-min} is an optimization problem over $\Delta_K$, the $K$-dimensional probability simplex, which is computationally complex and can be solved by using algorithmic techniques in \cite{neely2010dynamic}. However, as it is shown in Proposition \ref{prop:offline-determ} in Appendix~\ref{app:proof_pp2}, the optimal solution in our $K$-armed bandit setting is deterministic given the history $\mathcal{F}_{n-1}$. This allows us to define the offline Lyapunov-Drift Minimizing Policy $\piOff$ as in Algorithm~\ref{alg:LyOff}.

\emph{Intuition:} The policy $\piOff$ makes a balanced choice between the reward maximization and satisfying the constraints. For small $Q_n$, the controller selects the arm \yz{$I_n$} with the highest drift-plus-penalty ratio so as to maximize the expected total reward under the budget constraints. If $Q_{n}$ is large, then it means the constraint has been violated considerably, thus \yz{$I_n$}  is selected so as to reduce the penalty rate and hence violation level. Next, we prove finite-time performance bounds for $\piOff$. 

\begin{proposition}[Performance Bounds for $\piOff$] \label{prop:offline-opt}
    Suppose that Assumption \ref{assn:slater} and Assumption \ref{assn:bounded-moment} hold with positive $\epsilon$, $\sigma^2$, $r_{\tt max}$ and $\mu_{\tt min}$. Then, given the budget $B$, for any $V>0$ and $\delta\in[0,c)$, the regret and constraint-violation levels under $\piOff$ satisfy:
    \begin{align}
        {\tt REG}^{\piOff}(B) &= O\Bigg(\frac{\sigma^2 B}{V\mu_{\tt min}^2}+\frac{\delta r_{\tt max} B}{\epsilon \mu_{\tt min}^2}\Bigg),\\ D^{\piOff}(B) &= O\Big(\frac{1}{B} + \frac{V r_{\tt max}}{B\mu_{\tt min}\epsilon}-\frac{\delta}{\mu_{\tt min}}\Big).
    \end{align}

\yz{Specifically, let $V = v_0\sqrt{B}$ and $\delta = {\delta_0}/{\sqrt{B}}$ with some design parameters $\delta_0>0, v_0 > 0$. We can select $\delta_0 \in (\frac{r_{\tt max}}{\epsilon}v_0,c\sqrt{B})$ such that for sufficiently large $B$, 
\begin{equation}
    {\tt REG}^{\piOff}(B) = O\Big(\sqrt{B}\Big), \quad  D^{\piOff}(B) = O\Big( \frac{-1}{\sqrt{B}} \Big).
\end{equation}
}
\end{proposition}

The proof of Proposition \ref{prop:offline-opt} can be found in Appendix \ref{app:offline}. Proposition~\ref{prop:offline-opt} establishes the fact that $\piOff$ policy achieves $O(\sqrt{B})$ regret and \yz{zero} constraint-violation for \yz{sufficiently large} $B$ given the first-order statistics $\bE[X_{n,k}],\bE[R_{n,k}],\bE[{Y}_{n,k}]$ for all arms $k\in\bK$. The $\piOff$ policy will serve as a guide for our online learning algorithm, introduced next.

\begin{algorithm}[H]
\caption{$\tt LyOff$ Algorithm}\label{alg:LyOff}
\begin{algorithmic}[1]
\State \textbf{Input:} $B, K, c,  V, \delta$, \\
\qquad \quad $\bE[X_{1,k}], \bE[R_{1,k}], \bE[Y_{1,k}]$
\State Initialize $Q_0 = 0$, $\tt cost = 0$, $n = 1$
 \While{$\tt cost\leq B$}
  \State $\Psi_n(k, Q_n) = -V\frac{\bE[R_{1,k}]}{\bE[X_{1,k}]}+Q_{n}\frac{\bE[Y_{1,k}]}{\bE[X_{1,k}]}$
  \vspace{2pt}
  \State $k_n =\arg\min_{k\in \bK}~\Psi_n(k, Q_n) $
  \State Select arm $I_n = k_n$.
  \State Observe $X_n, R_n, Y_n$.
  \State $ Q_{n+1} = \max\big\{0, Q_{n} + Y_{n}-(c-\delta)X_n\big\}$
  \State $\tt cost = cost $ $+$ $X_n$.
  \State $n=n+1$.
 \EndWhile
\end{algorithmic} 
\end{algorithm}

\subsection{Online Lyapunov-Drift Minimizing Policy {\tt LyOn} Design}
\label{sec:LyOn}

A strong assumption in $\piOff$ was the a priori knowledge of the first-order statistics for all variables. Recall that in the learning problem, we do not have this knowledge. Instead, we must work with estimations by using the observed outcomes from bandit type feedback to learn the optimal decision. Furthermore, like all exploration-exploitation problems, the \textit{online exploration} is a crucial component of the learning problem here as well. Optimizing this trade-off with low regret and constraint-violation is particularly challenging in this setting due to the knapsack-type budget constraints from random costs, as well as the random penalties in the constraint.
In this section, we will design and analyze the {\tt LyOn} Algorithms by combining tools from renewal theory, stochastic control, as well as bandit optimization to address these challenges for optimal learning.

\emph{Strategy:} Our strategy will be to approximate the Lyapunov drift-plus-penalty ratio $\Psi_n$ in equation \eqref{eqn:lyapunov-dppr} by using the empirical estimates for the first-order statistics. In order to encourage online exploration, we will use confidence bounds so that the index at the end will be a high-probability lower bound for $\Psi_n$. The following definitions will be needed to define the online algorithm.
\begin{definition} [Confidence Radius] \label{def:cradius}
For any $n \geq 1$ and arm $k\in\bK$, let $\mathcal{I}_n^\pi(k)=\{t\in[1,n]: I_t^\pi = k\}$, $T_k^\pi(n) = |\mathcal{I}_n^\pi(k)| = \sum_{t=1}^n\I\{I_t^\pi = k\}$ be the number of pulls for arm $k$ under a policy $\pi$ in the first $n$ epochs. For a given $\alpha>0$, the confidence radius for arm $k$ is defined as:
${\tt rad}_k(n,\alpha) = \sqrt{\frac{2\alpha\log(n)}{T_k(n)}}.$
\end{definition}

To ensure the confidence radius is small enough, we have an initial exploration phase that is controlled by a parameter $\beta_0$ which depends on $\epsilon, y_{\tt max}$, and $\mu_{\tt min}$. Specifically, we set $\beta_0 = \frac{32\alpha(1+y_{\tt max})^2}{\mu_{\tt min}^2\epsilon^2}$ to guarantee the concentration event in Lemma~\ref{lemma:drift-bound} of Appendix~\ref{app:proof_pp2}.

For a subset of indices $S \subset\mathbb{N}$ and a stochastic process $\{Z_n:n\in\mathbb{N}\}$, let \yz{$\widehat{\bE}_S[Z] = \min\big\{1, \frac{1}{|S|}\sum_{t\in S}Z_t\big\},$} be the empirical mean estimator. Then, the empirical reward rate and empirical penalty rate under policy $\pi$ after $n$ epochs are defined as:
\begin{align}
    \widehat{r}_{n,k}^\pi = \frac{\widehat{\bE}_{\mathcal{I}_n^\pi(k)}[R_k]}{\widehat{\bE}_{\mathcal{I}_n^\pi(k)}[X_k]},\quad \widehat{y}_{n,k}^\pi = \frac{\widehat{\bE}_{\mathcal{I}_n^\pi(k)}[Y_{k}]}{\widehat{\bE}_{\mathcal{I}_n^\pi(k)}[X_k]}, \quad k\in \bK.
\end{align}

\begin{definition}[Empirical Lyapunov Drift-Plus-Penalty Ratio]
Let $Q_n^\pi$ be the variable evolving under $\pi$ as in \eqref{eqn:q-length}. Then, the empirical Lyapunov drift-plus-penalty ratio at epoch $n$ is defined as follows:
\begin{equation}
    \widehat{\Psi}_n(k, Q_n^\pi) = -V\cdot \widehat{r}_{n-1,k}^\pi + Q_{n}^\pi\cdot \widehat{y}_{n-1,k}^\pi,
\end{equation}
where $V>0$ is a design parameter. Define the empirical lower confidence bound for $\widehat{\Psi}_n(k, Q_n^\pi)$ as
\begin{align}\label{eqn:ly-index}
    \widehat{\Gamma}_n(k,Q_n^{\pi}) &= \widehat{\Psi}_n(k, Q_n^\pi)
     -{\tt rad}_k(n-1,\alpha)\frac{V(1+\widehat{r}_{n-1,k}^\pi)}{\widehat{\bE}_{\mathcal{I}_{n-1}^\pi(k)}[X_k]}\notag\\
     &\quad +{\tt rad}_k(n-1,\alpha)\frac{Q_{n}^{\pi}(1+\widehat{y}_{n-1,k}^\pi)}{\widehat{\bE}_{\mathcal{I}_{n-1}^\pi(k)}[X_k]}
\end{align}
\end{definition}
With these definitions, the online Lyapunov-Drift Minimizing Algorithm $\tt LyOn$ is defined in Algorithm~\ref{alg:LyOn}.

\begin{algorithm}[H]
\caption{$\tt LyOn$ Algorithm}\label{alg:LyOn}
\begin{algorithmic}[1]
\State \textbf{Input:} $B, K, c, \alpha, V, \delta, \beta_0, \mu_{\tt min}$
\State Initialize $Q_0 = 0$, $\tt cost = 0$, $n = 1$ 
\State Select each arm $\big\lceil \beta_0 \log\big(\frac{2B}{\mu_{\tt min}}\big)\big\rceil$ times.
\State Update $n$, $\tt cost$, $\widehat{\Gamma}_n(k,Q_n)$ (eq.~\eqref{eqn:ly-index}).
 \While{$\tt cost\leq B$}
   \State $k_n = {\arg\min}_{k\in \bK}\big\{\widehat{\Gamma}_n(k,Q_n)\big\}$
  \State Select arm $I_n=k_n$. Observe $X_n, R_n, Y_n$.
  \State $ Q_{n+1} = \max\big\{0, Q_{n} + Y_{n}-(c-\delta)X_n\big\}$
  \State $\tt cost = cost $ $+$ $X_n$.
  \State Update $\widehat{\Gamma}_n(k,Q_n)$ (eq.~\eqref{eqn:ly-index}).
  \State $n=n+1$.
 \EndWhile
\end{algorithmic} 
\end{algorithm}

\begin{remark}\normalfont
Before we analyze it, we make the following observations about the {\tt LyOn} Algorithm.
\begin{enumerate}
    \item {\tt Lyon} is an extremely low-complexity, iterative algorithm, whereby in every step a simple update is performed. 
    \item The index to be minimized in \eqref{eqn:ly-index} is a high-probability lower bound for $\Psi_n(k,Q_n^\piLy)$. Thus, given the available data $\mathcal{F}_{n-1}^\piLy$, the algorithm makes an optimistic drift-minimizing arm selection in the face of uncertainty.

    \item If $I_n^\piLy = k$, then at least one of the following must be true: a) High confidence for arm $k$, large $r_k$ and small $Q_{n}^\piLy$. b) High confidence for arm $k$, large $Q_{n}^\piLy$ and small $y_{k}$. c) Low confidence for arm $k$. As such, the {\tt LyOn} Algorithm incentivizes online exploration by choosing arms with very low confidence.
    
    \item The {\tt LyOn} Algorithm extends the {\tt UCB-BwI} Algorithm proposed in \cite{cayci2019learning} to the non-trivial and useful cases with stochastic feasibility constraints. Note that $Q_{n}^\piLy = 0$ if there is no constraint, thus the {\tt LyOn} Algorithm reduces to the {\tt UCB-BwI} Algorithm.
    
\end{enumerate}
\end{remark}

\begin{theorem}[Performance Bounds for $\piLy$]\label{thm:regret}
Suppose that Assumption \ref{assn:slater} and Assumption \ref{assn:bounded-moment} hold with positive $\epsilon$, $\sigma^2$, $r_{\tt max}$, $y_{\tt max}$, and $\mu_{\tt min}$. Then, for any $V>0$ and $\delta\in[0,c)$, the regret and constraint-violation levels under $\piLy$ satisfy:
    \begin{align}
        {\tt REG}^{\piLy}(B) &= O \Bigg(\frac{r_{\tt max}\sqrt{\yz{K}B\log B}}{\mu_{\tt min}^2}+\frac{y_{\tt max}^2 K\log B}{\epsilon^2\mu_{\tt min}^2} \notag \\
        & +\frac{\sigma^2+y_{\tt max}+\delta \log B}{V\mu_{\tt min}^2}B+\frac{\delta r_{\tt max}}{\epsilon \mu_{\tt min}^2}B \Bigg),
    \end{align}
    \begin{align}
        D^{\piLy}(B) &= O\Bigg(\frac{y_{\tt max}^2K\log B}{\epsilon^2\mu_{\tt min}^2 B} + \frac{V r_{\tt max}}{B\mu_{\tt min}\epsilon}-\frac{\delta}{\mu_{\tt min}} \Bigg),
    \end{align}

\yz{Specifically, let $V = v_0\sqrt{B\log B}$  and $\delta = {\delta_0}{\sqrt{\log B /B}}$ with design parameters $\delta_0>0, v_0 > 0$. We can select $\delta_0 \in (\frac{r_{\tt max}}{\epsilon}v_0,c\sqrt{B})$ such that for sufficiently large $B$, 
\begin{equation}
    {\tt REG}^{\piLy}(B) = O\Big(\sqrt{\yz{K}B\log B}\Big), \; D^{\piLy}(B)  = O\Big(\frac{-1}{\sqrt{B}}\Big).
\end{equation}
}
\end{theorem}
The proof of Theorem~\ref{thm:regret} can be found in Appendix \ref{app:online_bound_proof}. Theorem \ref{thm:regret} implies that $\piLy$ achieves $O(\sqrt{\yz{K}B\log B})$ regret and \yz{zero} constraint-violation for \yz{sufficiently large} $B$ while learning the first order statistics under a bandit feedback.

In addition to the fact that cumulative reward and penalty processes form stopped and controlled random walks, the main challenge in analyzing the $\tt LyOn$ algorithm performance is that $Q_n$ is correlated with the sample path. To address this, we prove a maximal inequality for $Q_n$ under a concentration event (Lemma~\ref{lemma:max-Q} in Appendix~\ref{app:online_bound_proof}), which can have its own value in other queuing systems.
Also note that, compared with $\tt LyOff$, the online algorithm has a very small increase on the regret bounds by a factor of $\sqrt{\yz{K}\log B}$. This is a reasonable price to pay since we are not assuming any known statistics. To the best of our knowledge, these are the best results available on both regret and constraint-violation in the current setup. In the special case of unit cost scenario, our algorithm theoretically guarantees a similar regret performance to prior designs \cite{agrawal2014bandits} while providing a stronger constraint-violation guarantee.

\begin{figure*}[htp]
\centering
\subfloat[Reward rate ($\tt REW(B)/B)$]{\includegraphics[width=2in]{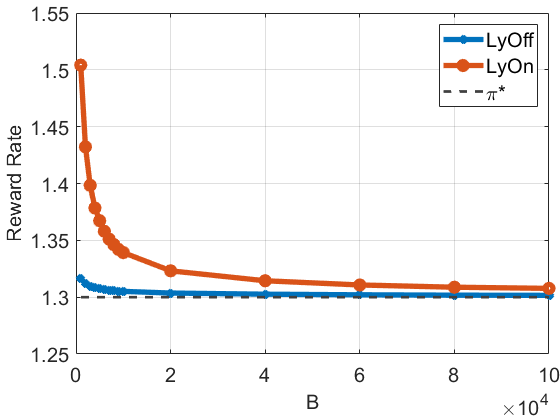}
\label{fig:Reward_Rate_0_5}}
\hfil
\subfloat[Constraint-violation]{\includegraphics[width=2in]{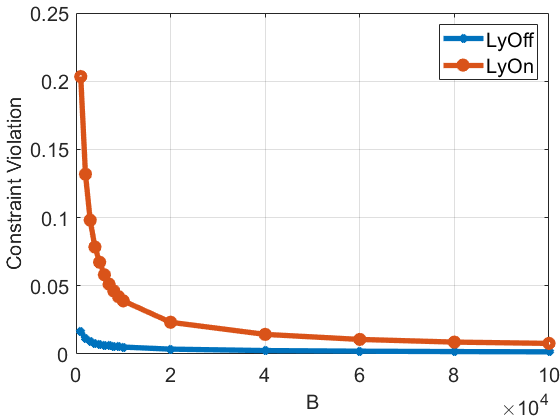}
\label{fig:Constraint_Violation_0_5}}
\hfil
\subfloat[Time allocation]{\includegraphics[width=2in]{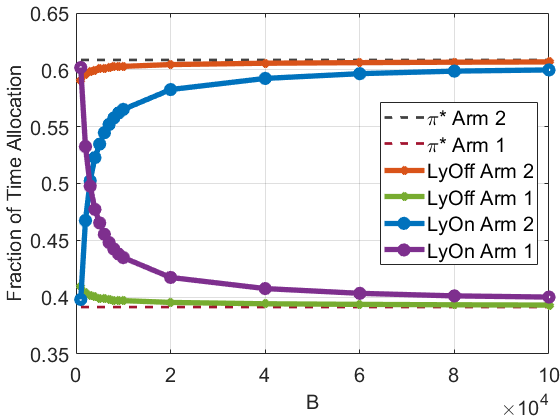}
\label{fig:Time_Allocation_0_5}}

\subfloat[Reward rate ($\tt REW(B)/B)$]{\includegraphics[width=2in]{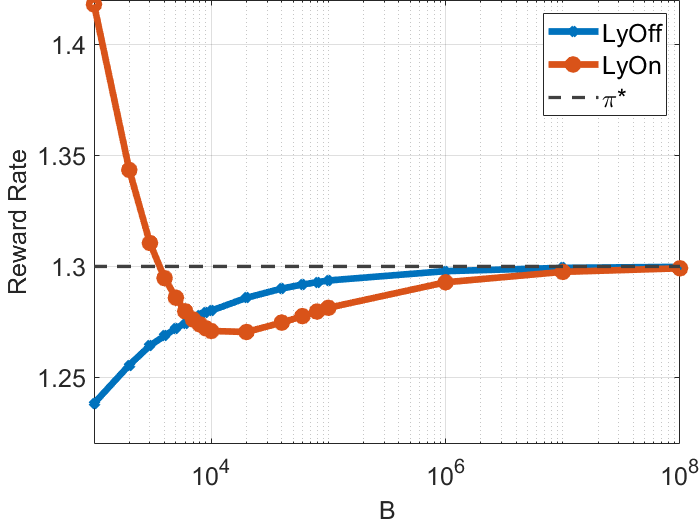}
\label{fig:Reward_Rate_3}}
\hfil
\subfloat[Constraint-violation]{\includegraphics[width=2in]{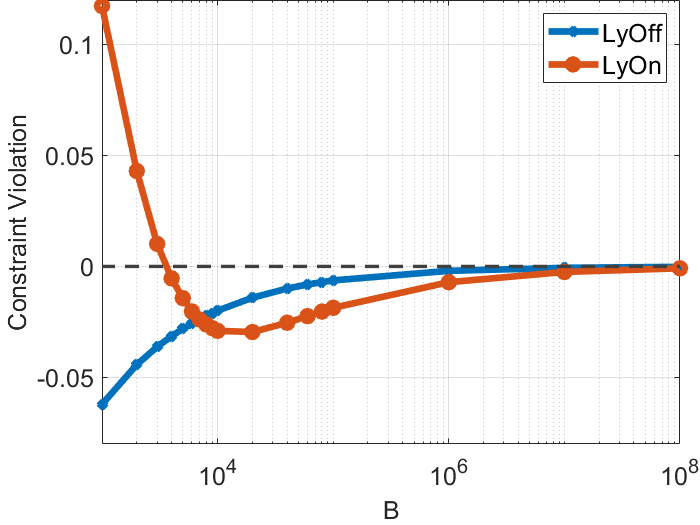}
\label{fig:Constraint_Violation_3}}
\hfil
\subfloat[Time allocation]{\includegraphics[width=2in]{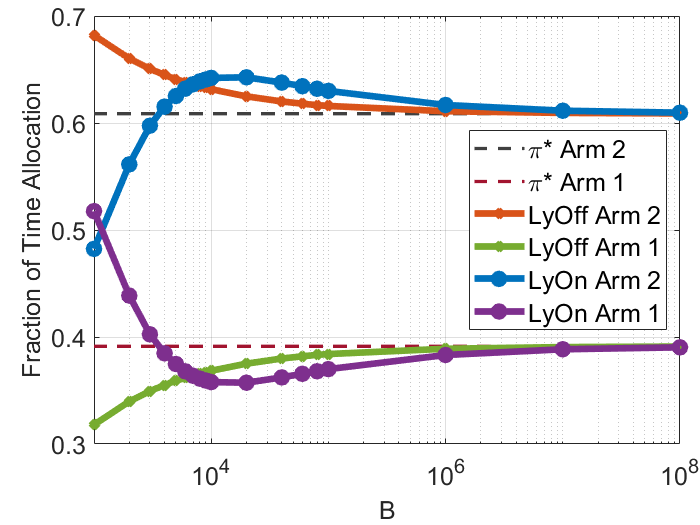}
\label{fig:Time_Allocation_3}}

\caption{Performance of $\tt LyOff$ and $\tt LyOn$ for different parameters. (a), (b), and (c) use $v_0=1$, $\delta_0 = 0.5$, (d), (e), and (f) use $v_0=1$, $\delta_0=15$.}
\label{fig:simulation_main}
\end{figure*}

\section{Simulations}\label{sec:simulations}

We implement both $\tt LyOff$ and $\tt LyOn$ algorithms for $K=2$ arms with Bernoulli distributed rewards, costs, and penalties. Assuming $c = 0.8$, arm 1 is selected to have a high reward rate and
a high penalty rate with $\bE[X_1] = 0.4, \bE[Y_1] = 0.6$, and $\bE[R_1] = 0.8$. Arm 2 is selected to have a low reward rate and a low penalty rate with $\bE[X_2] = 0.6, \bE[Y_2] = 0.3$, and $\bE[R_2] = 0.6$. These values are interesting in that, an optimal controller will have to make a trade-off between the two arms, whereas any static policy selecting one of the arms will result in either linear regret or linear constraint-violation.

Figure~\ref{fig:Reward_Rate_0_5}, \ref{fig:Constraint_Violation_0_5}, and \ref{fig:Time_Allocation_0_5} show the simulation results (averaged over $10^4$ runs) with $v_0=1$  and $\delta_0=0.5$ for $\tt LyOff$ and $\tt LyOn$ algorithms. 
To observe the reward rate behavior, in Figure~\ref{fig:Reward_Rate_0_5}, we plot the reward rates ${\tt REW}^{\pi}(B)/B$ of $\piOff$ and $\piLy$ and the optimal randomized policy $\pi^*$, with varying budgets $B$. 
This figure shows that both the offline and the online designs reach the rate of the optimal design, as predicted by our analysis. Also, Figure~\ref{fig:Constraint_Violation_0_5} verifies the fast decaying of constraint-violation with rate $\Tilde{O}(1/B)$ as $B$ increases, which confirms the scaling behaviour revealed in our analyses. Figure~\ref{fig:Time_Allocation_0_5} further confirms the convergence of $\tt LyOff$ and $\tt LyOn$ towards $\pi^*$ by showing the proportion of time allocated to each arm. \yz{As predicted by Theorem~\ref{thm:regret}, Figure~\ref{fig:Reward_Rate_3}, \ref{fig:Constraint_Violation_3}, and \ref{fig:Time_Allocation_3} show that we can indeed select specific $v_0$ and $\delta_0$ values such that the constraint-violation becomes negative when $B$ is sufficiently large. At the same time, the reward rate and proportion of time allocated to each arm still converge to the rate of the optimal design.}

In Appendix~\ref{app:simulations}, to check the performance of our algorithms for larger $K$, we increase the number of arms by adding arms with the principle that high reward rate arm also has high penalty rate (otherwise the arms are not competitive). We also investigate the effect of design choices $V$ and $\delta$ to capture the tradeoff between constraint-violation and regret under the $\tt LyOff$ and $\tt LyOn$ algorithms.

\section{Conclusion}
In this paper, we proposed a broadly applicable computationally efficient methodology based on Lyapunov-drift-minimization for solving a penalty-constrained reward maximization problem with a limited budget, random costs, and bandit feedback. Both offline and online algorithms are developed based on this design methodology, which are also proven to have sharp regret and
constraint-violation performance. The approach and algorithms are applicable in diverse domains whereby knapsack budget constraints and stochastic feasibility constraints are required. 
An interesting future work that can benefit from the same methodology would be to extend our setting to the scenario of \yz{multiple constraints} and infinitely many arms.


\section*{Acknowledgments}
This work is supported in part by the NSF grants: CNS-NeTS-1717045, CNS-SpecEES-1824337, CNS-NeTS-2007231, CNS-NeTS-2106679, IIS-2112471, CCF-1934986; and the ONR Grant N00014-19-1-2621.

\bibliography{aaai22}

\newpage
\appendix
\onecolumn
\section{Applications}\label{app:applications}

\paragraph{$\bullet$ Contractual Hiring }

Our problem formulation also finds interesting applications in fair contractual hiring \cite{hannak2017bias,ipeirotis2010analyzing} and server allocation \cite{harchol2000task}. Consider a task allocation problem, where tasks are sequentially allocated to a worker/server who belongs to one of $K$ groups. If $n^{th}$ task is allocated to a worker from group $k$, it takes $X_{n,k}$ time to complete the task, and yields a reward $R_{n,k}$ upon completion. For reward maximization objective, the tasks are allocated to the worker group that maximizes the expected total reward, and the other groups do not receive any tasks (see \cite{cayci2020group}). However, in many systems, the "fair" allocation of tasks are desired. One way to impose fairness is to impose constraints on the total time budget $B$ allocated to user groups. Let $A\subset \mathbb{K}$ be a class of user groups. By letting $Y_{n,k} =\omega\cdot \I\{k\notin A\}$ for all $k$ and some $\omega \geq 0$, the solution to the optimization problem \eqref{eqn:optimization-problem} yields a fair allocation rule where the fraction of the budget allocated to the worker groups in $A$ should be at least proportional to $c$.

\paragraph{$\bullet$ Online Advertising}
The same idea can be applied to online advertising. Consider an advertiser with a budget $B$ and $K$ advertising spots (\textit{impressions}) to select from. In order to win an impression, the advertiser has to bid in an auction \cite{ghosh2009bidding, balseiro2019learning}, which means the budget consumption $X_{n,k}$ is random. Winning an impression will generate a random return $R_{n,k}$ which measures the value of that impression. The goal of an advertiser is to maximize the expected total return under the budget constraint. In addition, the advertiser may want to set a constraint on the number of bids for certain type of impressions to balance between different demographic groups and locations, which can be represented by $Y_{n,k}$ defined above.

\section{Preliminary Results}\label{app:preliminary_results}
\subsection{Results from Renewal Theory}\label{app:renewal_theory}
In this section, we will prove results on the counting process $N^\pi(B)$ based on renewal theory, which will be used throughout the proofs in the following sections. The proofs are based on \cite{cayci2020budget}.

\begin{proposition}[High-probability upper bound for $N^\pi(B)$]
Consider an i.i.d. process $X_{n,k}\in[0,1]$ a.s. for each $k\in\bK$, and define $n_0(B)=\lceil 2B/\mu_{\tt min}\rceil$ for $\mu_{\tt min}=\min_{k\in\bK}~\bE[X_{n,k}]$. Then, under any policy $\pi\in\Pi$, we have the following inequality for all $B>1$:
\yz{
\begin{equation}
    \bP(N^\pi(B) \geq n) = \bP\Big(\sum_{t=1}^nX_t^\pi \leq B\Big) \leq e^{-n\mu_{\tt min}^2/8},
    \label{eqn:renewal-concentration}
\end{equation}
}
\noindent for any $n \geq n_0(B)$.
\label{prop:renewal-concentration}
\end{proposition}
\begin{proof}
The equality in \eqref{eqn:renewal-concentration} is due to the renewal relation $\{N^\pi(B)\geq n\} = \Big\{\sum_{t=1}^n X_t^\pi \leq B\Big\}.$ In order to prove the inequality, first note that $\bE[X_n^\pi|\mathcal{F}_{n-1}^\pi] \geq \mu_{\tt min}$ for any $n\geq 1$, and $D_n = X_n^\pi - \bE[X_n^\pi|\mathcal{F}_{n-1}^\pi]$ is a martingale difference sequence with $|D_n| \leq 1$ almost surely. Then, we have the following relation: 
\begin{align*}\{S_n^\pi \leq B\} &= \Big\{\sum_{t=1}^nD_t \leq B-\sum_{t=1}^n\bE[X_{t}^\pi|\mathcal{F}_{t-1}^\pi]\Big\} \subset \{\sum_{t=1}^nD_t \leq B-n\mu_{\tt min}\}.
\end{align*}
For any $n \geq n_0(B)$, we have $B-n\mu_{\tt min} \leq -n\mu_{\tt min}/2$. \yz{Therefore, we have:
\begin{align*}
    \bP(S_n^\pi \leq B) &\leq \bP\Big(\sum_{t=1}^nD_t \leq -\frac{n\mu_{\tt min}}{2}\Big),\\
    &\leq e^{-n\mu_{\tt min}^2/8},
\end{align*}}
where the second inequality follows from concentration bounds for martingale difference sequence \yz{(Corollary 2.20 in \cite{wainwright2019high}) with $|D_n|\leq 1$ and deviation $t = \mu_{\tt min}/2$} .
\end{proof}

In the following, we characterize the performance of a stationary randomized policy $\pi = \pi(\mathbf{p})$ by proving tight bounds on the expected cumulative reward and penalty under $\pi(\mathbf{p})$.

\begin{proposition}[Performance of a Stationary Randomized Policy]\label{prop:srp}
    For any $\mathbf{p}\in\Delta_K$, let $\pi(\mathbf{p})$ be a stationary randomized policy, and let $r(\mathbf{p})=\frac{\sum_{k\in\bK}p_k\bE[R_{1,k}]}{\sum_{k\in\bK}p_k\bE[X_{1,k}]}$ and $y(\mathbf{p})=\frac{\sum_{k\in\bK}p_k\bE[Y_{1,k}]}{\sum_{k\in\bK}p_k\bE[X_{1,k}]}$. Then, we have the following inequalities:
    \begin{equation}
       r(\mathbf{p}) B \leq \bE[{\tt REW}^{\pi(\mathbf{p})}(B)] \leq r(\mathbf{p})\Big(B+\frac{1}{\mu_{\tt min}^2}\Big).
    \end{equation}
    
    \begin{equation}
       y(\mathbf{p}) B \leq  \bE\Bigg[\sum_{n=1}^{N^{\pi(\mathbf{p})}(B)}Y_{n}^{\pi(\mathbf{p})}\Bigg]\leq y(\mathbf{p})\Big(B+\frac{1}{\mu_{\tt min}^2}\Big).
    \end{equation}
\end{proposition}

\begin{proof}
    The proof is based on Proposition 2 in \cite{cayci2020group}. Under $\pi(\mathbf{p})$, it can be shown that arm $k$ is pulled exactly once in a regenerative cycle with total expected cost $\big(\sum_jp_j\bE[X_{1,j}]\big)/p_k$. Therefore, the reward per unit cost from arm $k$ under $\pi(\mathbf{p})$ is  $r_k(\mathbf{p})=\frac{p_k\bE[R_{1,k}]}{\sum_jp_j\bE[X_{1,j}]}$, which implies the reward per unit cost from all arms is $r(\mathbf{p})=\sum_kr_k(\mathbf{p})$. Similarly, the penalty per unit cost from arm $k$ under $\pi(\mathbf{p})$ is $y_k(\mathbf{p})=\frac{p_k\bE[Y_{1,k}]}{\sum_jp_j\bE[X_{1,j}]}$, which implies the penalty per unit cost from all arms is $y(\mathbf{p})=\sum_ky_k(\mathbf{p})$. Hence, the upper and lower bounds follow from Proposition 6.1 and Proposition 6.2 (Lorden's inequality) in section V of \cite{asmussen2008applied}, respectively.
\end{proof}

\subsection{Decomposition of Regret and Constraint-violation}\label{app:regret_decomp}
\begin{lemma}[Regret Decomposition] \label{lemma:regret-decomp}
For any causal policy $\pi\in\Pi$, the regret with respect to the optimal stationary policy $\pi^*$ can be decomposed as:
\begin{align*}
     \bE[{\tt REW}^{\pi^*}(B)]-\bE[{\tt REW}^{\pi}(B)] &\leq \Big(r(\mathbf{p}^*)- \frac{\bE[\sum_{n=1}^{n_0(B)}R_n^\pi]}{\bE[\sum_{n=1}^{n_0(B)} X_n^\pi]}\Big)\cdot \frac{2B}{\mu_{\tt min}}  +\frac{r_{\tt max}}{\mu_{\tt min}^2}\Big(1+8e^{-B\mu_{\tt min}/4}\Big)
\end{align*}
where $\mathbf{p^*} = (p_1^*, \ldots, p_K^*)$ is the distribution of budget under $\pi^*$ and $n_0(B)=\lceil 2B/\mu_{\tt min}\rceil$ is a high probability upper bounds for $N^\pi(B)$.
\end{lemma}

\begin{proof}
Take an arbitrary causal policy $\pi\in\Pi$. Since $\pi$ is causal, the following holds for some $\mathbf{p}_n^\pi=(p_{1,n}^\pi,p_{2,n}^\pi,\ldots,p_{K,n}^\pi)\in\mathcal{F}_{n-1}^\pi$: \begin{equation}\label{eqn:reward-rate-filtration}
    \bE[R_{n}^\pi|\mathcal{F}_{n-1}^\pi] =r(\mathbf{p}_n^\pi)\bE[X_{n}^\pi|\mathcal{F}_{n-1}^\pi]=r(\mathbf{p}_n^\pi)\sum_{k\in\bK}p_{k,n}^\pi\bE[X_{n,k}|\mathcal{F}_{n-1}^\pi].
\end{equation}

 We have the following inequality for the expected cumulative reward under $\pi$:
\begin{align}
    \nonumber \bE[{\tt REW}^\pi(B)] &= \bE\Big[\sum_{n=1}^\infty \I\{S_{n-1}^\pi \leq B\}R_{n}^\pi\Big], \\ \label{eqn:wald-b} &= \bE\Big[\sum_{n=1}^\infty\bE\big[R_{n}^\pi|\mathcal{F}_{n-1}^\pi\big]\I\{S_{n-1}^\pi \leq B\}\Big], \\
    &= \label{eqn:wald-c} \bE\Big[\sum_{n=1}^\infty r(\mathbf{p}_n^\pi) \bE\big[X_{n}^\pi|\mathcal{F}_{n-1}^\pi\big]\I\{S_{n-1}^\pi \leq B\}\Big],
\end{align}
where \eqref{eqn:wald-b} follows since $\pi$ is causal (i.e., $I_n^\pi\in\mathcal{F}_{n-1}^\pi$) and $S_{n-1}^\pi\in\mathcal{F}_{n-1}$, and \eqref{eqn:wald-c} follows from the relation \eqref{eqn:reward-rate-filtration}.

For the optimal stationary randomized policy $\pi^*$, by Proposition \ref{prop:srp}, we have the following inequalities:
\begin{align}
    \bE[{\tt REW}^{\pi^*}(B)] &\leq r(\mathbf{p}^*)\Big(B+\frac{1}{\mu_{\tt min}^2}\Big), \notag\\
    &\leq r(\mathbf{p}^*)\Big(\bE\Big[\sum_{n=1}^{N^\pi(B)}X_n^\pi\Big]+\frac{1}{\mu_{\tt min}^2}\Big), \label{eqn:pi_star_up_b}\\
    &= \bE\Big[\sum_{n=1}^\infty r(\mathbf{p}^*) \bE\big[X_{n}^\pi|\mathcal{F}_{n-1}^\pi\big]\I\{S_{n-1}^\pi \leq B\}\Big]+\frac{r(\mathbf{p}^*)}{\mu_{\tt min}^2}, \label{eqn:pi_star_up_c}
\end{align}
where \eqref{eqn:pi_star_up_b} uses the fact that $B \leq \sum_{n=1}^{N^\pi(B)}X_n^\pi$.

Combining \eqref{eqn:wald-c} and \eqref{eqn:pi_star_up_c}, we have 
\begin{align}
    \bE[{\tt REW}^{\pi^*}(B)]-\bE[{\tt REW}^{\pi}(B)] &\leq \bE\Big[\sum_{n=1}^\infty \big(r(\mathbf{p}^*)-r(\mathbf{p}_n^\pi)\big) \bE\big[X_{n}^\pi|\mathcal{F}_{n-1}^\pi\big]\I\{S_{n-1}^\pi \leq B\}\Big]+\frac{r(\mathbf{p}^*)}{\mu_{\tt min}^2}\notag\\
    &\leq \bE\Big[\sum_{n=1}^N \big(r(\mathbf{p}^*)-r(\mathbf{p}_n^\pi)\big) \bE\big[X_{n}^\pi|\mathcal{F}_{n-1}^\pi\big]\I\{S_{n-1}^\pi \leq B\}\Big]\notag\\
    &\quad + r_{\tt max}\Big(\frac{1}{\mu_{\tt min}^2}+\sum_{n>N}\bP(S_{n-1}^\pi \leq B)\Big). \label{eqn:pi_star_gap_c}
\end{align}
where \eqref{eqn:pi_star_gap_c} uses the fact that $X_n^\pi \in [0,1]$ and $r(\mathbf{p}^*)\leq r_{\tt max}$.

Let $n_0(B)=\lceil 2B/\mu_{\tt min}\rceil$. \yz{By Proposition \ref{prop:renewal-concentration}, $\bP(S_{n-1}^\pi \leq B) \leq e^{-n\mu_{\tt min}^2/8}$ for all $n > n_0(B)$, which implies: 
\begin{align*}
    \sum_{n > n_0(B)}\bP(S_{n-1}^\pi \leq B) \leq \frac{8e^{-B\mu_{\tt min}/4}}{\mu_{\tt min}^2}.
\end{align*}
}
Thus, setting $N=n_0(B)$, we have the following inequality:
\begin{align*}
    \bE[{\tt REW}^{\pi^*}(B)]-\bE[{\tt REW}^{\pi}(B)] &\leq r(p^*)\bE\Big[\sum_{n=1}^{n_0(B)} X_n^\pi\Big] - \bE\Big[\sum_{n=1}^{n_0(B)}R_n^\pi\Big] + r_{\tt max}\Big(\frac{1}{\mu_{\tt min}^2}+\frac{8e^{-B\mu_{\tt min}/4}}{\mu_{\tt min}^2}\Big)\\
    &= \Big(r(p^*)- \frac{\bE[\sum_{n=1}^{n_0(B)}R_n^\pi]}{\bE[\sum_{n=1}^{n_0(B)} X_n^\pi]}\Big)\cdot \bE\Big[\sum_{n=1}^{n_0(B)} X_n^\pi\Big] \\
    &\quad + \frac{r_{\tt max}}{\mu_{\tt min}^2}\Big(1+8e^{-B\mu_{\tt min}/4}\Big)\\
    &\leq \Big(r(p^*)- \frac{\bE[\sum_{n=1}^{n_0(B)}R_n^\pi]}{\bE[\sum_{n=1}^{n_0(B)} X_n^\pi]}\Big)\cdot \frac{2B}{\mu_{\tt min}}  + \frac{r_{\tt max}}{\mu_{\tt min}^2}\Big(1+8e^{-B\mu_{\tt min}/4}\Big)
\end{align*}

\end{proof}

\begin{lemma}[Constraint-violation Decomposition] \label{lemma:constraint-decomp}
For any causal policy $\pi\in\Pi$, the constraint-violation can be decomposed as:
\begin{align*}
     D^{\pi}(B) &\leq \Bigg( \frac{\bE\Big[\sum_{n=1}^{n_0(B)}(Y_n^\pi)\Big]}{\bE\Big[\sum_{n=1}^{n_0(B)}(X_n^\pi)\Big]}-c\Bigg)\cdot \frac{2}{\mu_{\tt min}} + \frac{8(y_{\tt max}-c)}{B\mu_{\tt min}^2}e^{-B\mu_{\tt min}/4} +\frac{1}{B}
\end{align*}
\end{lemma}

\begin{proof} For any causal policy $\pi$, the following holds for some $\mathbf{p}_n^\pi=(p_{1,n}^\pi,p_{2,n}^\pi,\ldots,p_{K,n}^\pi)\in\mathcal{F}_{n-1}^\pi$: \begin{equation}\label{eqn:penalty-rate-filtration}
    \bE[Y_{n}^\pi|\mathcal{F}_{n-1}^\pi] =y(\mathbf{p}_n^\pi)\bE[X_{n}^\pi|\mathcal{F}_{n-1}^\pi]=y(\mathbf{p}_n^\pi)\sum_{k\in\bK}p_{k,n}^\pi\bE[X_{n,k}|\mathcal{F}_{n-1}^\pi].
\end{equation}

Therefore, the constraint-violation satisfies
\begin{align}
   B\cdot D^{\pi}(B) &= \mathbb{E}\Big[\sum_{n=1}^{N^\pi(B)}Y_{n}^\pi\Big]-Bc \notag\\
   &\leq \mathbb{E}\Big[\sum_{n=1}^{N^\pi(B)-1}(Y_{n}^\pi-cX_n^\pi)\Big]+1 \label{eqn:constraint_decomp_b}\\
    &= \bE\Big[\sum_{n=1}^\infty \bE\big[(Y_{n}^\pi-cX_n^\pi)|\mathcal{F}_{n-1}^\pi\big] \I\{S_{n-1}^\pi < B\}\Big]+1 \notag\\
    &= \bE\Big[\sum_{n=1}^\infty (y(\mathbf{p}_n^\pi)-c) \bE\big[X_{n}^\pi|\mathcal{F}_{n-1}^\pi\big]\I\{S_{n-1}^\pi < B\}\Big]+1 \label{eqn:constraint_decomp_d}\\
    &\leq \bE\Big[\sum_{n=1}^N (y(\mathbf{p}_n^\pi)-c) \bE\big[X_{n}^\pi|\mathcal{F}_{n-1}^\pi\big]\I\{S_{n-1}^\pi < B\}\Big]\notag\\
    &\quad +(y_{\tt max}-c)\sum_{n>N}\bP(S_{n-1}^\pi \leq B) +1 \label{eqn:constraint_decomp_e}
\end{align}
where \eqref{eqn:constraint_decomp_b} uses the fact that $B \geq \sum_{n=1}^{N^\pi(B)-1}X_n^\pi$ and $Y_n^\pi-cX_n^\pi \leq 1$; \eqref{eqn:constraint_decomp_d} uses the result in \eqref{eqn:penalty-rate-filtration}; \eqref{eqn:constraint_decomp_e} uses the fact that $y(\mathbf{p}_n^\pi) \leq y_{\tt max}$.

Let $n_0(B)=\lceil 2B/\mu_{\tt min}\rceil$. Then, by Proposition \ref{prop:renewal-concentration},
\begin{align*}
   B\cdot D^{\pi}(B) &\leq \bE\Big[\sum_{n=1}^{n_0(B)}(Y_n^\pi)\Big] - c \bE\Big[\sum_{n=1}^{n_0(B)}(X_n^\pi)\Big] + \frac{8(y_{\tt max}-c)}{\mu_{\tt min}^2}e^{-B\mu_{\tt min}/4}+1\\
   & = \Big( \frac{\bE\Big[\sum_{n=1}^{n_0(B)}(Y_n^\pi)\Big]}{\bE\Big[\sum_{n=1}^{n_0(B)}(X_n^\pi)\Big]}-c\Big)\cdot \bE\Big[\sum_{n=1}^{n_0(B)}(X_n^\pi)\Big] + \frac{8(y_{\tt max}-c)}{\mu_{\tt min}^2}e^{-B\mu_{\tt min}/4}+1\\
   &\leq \Big( \frac{\bE\Big[\sum_{n=1}^{n_0(B)}(Y_n^\pi)\Big]}{\bE\Big[\sum_{n=1}^{n_0(B)}(X_n^\pi)\Big]}-c\Big)\cdot \frac{2B}{\mu_{\tt min}} + \frac{8(y_{\tt max}-c)}{\mu_{\tt min}^2}e^{-B\mu_{\tt min}/4}+1
\end{align*}

Therefore, 
\begin{align*}
     D^{\pi}(B) &\leq \Big( \frac{\bE\Big[\sum_{n=1}^{n_0(B)}(Y_n^\pi)\Big]}{\bE\Big[\sum_{n=1}^{n_0(B)}(X_n^\pi)\Big]}-c\Big)\cdot \frac{2}{\mu_{\tt min}} + \frac{8(y_{\tt max}-c)}{B\mu_{\tt min}^2}e^{-B\mu_{\tt min}/4} +\frac{1}{B}
\end{align*}

\end{proof}

\section{Proof of Proposition \ref{prop:optimality-gap}}\label{app:proof_pp1}
\label{app:optimality-gap}


\begin{proof}
The proof consists of two parts. In the first part, we will prove that the regret under $\pi^*$ is $O(1)$. In the second part, we will prove that the constraint-violation under $\pi^*$ vanishes at a rate $O(1/B)$.

\paragraph{(Bounding the regret)} Consider the optimal stationary randomized policy $\pi^*$ as defined in Definition \ref{def:osrp}. By Proposition \ref{prop:srp}, for a causal policy $\pi$, we have the following inequalities:
\begin{align}
    \nonumber \bE[{\tt REW}^{\pi^*}(B)] &\geq r(\mathbf{p}^*)B,\\
    \label{eqn:wald-srp-a}&\geq r(\mathbf{p}^*)\Big(\bE\Big[\sum_{n=1}^{N^\pi(B)}X_n^\pi\Big]-1\Big),\\
    \label{eqn:wald-srp-b}&= \bE\Big[\sum_{n=1}^\infty r(\mathbf{p}^*) \bE\big[X_{n}^\pi|\mathcal{F}_{n-1}^\pi\big]\I\{S_{n-1}^\pi \leq B\}\Big]-r(\mathbf{p}^*),
\end{align}
where \eqref{eqn:wald-srp-a} follows since $\sum_{n=1}^{N^\pi(B)-1}X_n^\pi \leq B \leq \sum_{n=1}^{N^\pi(B)}X_n^\pi$ and $X_{n,k}\in[0,1]$ almost surely. 

Combining \eqref{eqn:wald-c} and \eqref{eqn:wald-srp-b}, we have:
\begin{equation}
    \bE[{\tt REW}^\pi(B)]-\bE[{\tt REW}^{\pi^*}(B)] \leq \bE\Big[\sum_{n=1}^\infty \big(r(\mathbf{p}_n^\pi)-r(\mathbf{p}^*)\big) \bE\big[X_{n}^\pi|\mathcal{F}_{n-1}^\pi\big]\I\{S_{n-1}^\pi \leq B\}\Big]+r(\mathbf{p}^*).
\end{equation}
Since $r(\mathbf{p})\leq r_{\tt max}$ for all $\mathbf{p}\in \Delta_K$, we have the following inequality for all $n>1$:
\begin{align}
    \bE[{\tt REW}^\pi(B)]-\bE[{\tt REW}^{\pi^*}(B)] \leq \bE\Big[\sum_{n=1}^N \big(r(\mathbf{p}_n^\pi)-r(\mathbf{p}^*)\big) \bE\big[X_{n}^\pi|\mathcal{F}_{n-1}^\pi\big]\I\{S_{n-1}^\pi \leq B\}\Big]\notag\\
    + r_{\tt max}\Big(1+\sum_{n>N}\bP(S_{n-1}^\pi \leq B)\Big).
    \label{eqn:opt-gap-a}
\end{align}

Let $n_0(B)=\lceil 2B/\mu_{\tt min}\rceil$. Then, by Proposition \ref{prop:renewal-concentration}, $\bP(S_{n-1}^\pi \leq B) \leq e^{-n\mu_{\tt min}^2/8}$ for all $n > n_0(B)$, which implies: 
\begin{align}
    \sum_{n > n_0(B)}\bP(S_{n-1}^\pi \leq B) \leq \frac{8e^{-B\mu_{\tt min}/4}}{\mu_{\tt min}^2}. \label{eqn:bound_s_n}
\end{align}

Thus, setting $N=n_0(B)$ in \eqref{eqn:opt-gap-a}, we have the following inequality:
\begin{align}
    \bE[{\tt REW}^\pi(B)]-\bE[{\tt REW}^{\pi^*}(B)] &\leq \bE\Big[\sum_{n=1}^{n_0(B)}R_n^\pi\Big]- r(\mathbf{p}^*)\bE\Big[\sum_{n=1}^{n_0(B)} X_n^\pi\Big] + r_{\tt max}\Big(1+\frac{8e^{-B\mu_{\tt min}/4}}{\mu_{\tt min}^2}\Big), \notag\\
   \nonumber &\leq \Big(\frac{\bE[\sum_{n=1}^{n_0(B)}R_n^\pi]}{\bE[\sum_{n=1}^{n_0(B)} X_n^\pi]}- r(\mathbf{p}^*)\Big)\cdot \bE\Big[\sum_{n=1}^{n_0(B)} X_n^\pi\Big] \\
   &\quad + r_{\tt max}\Big(1+\frac{8e^{-B\mu_{\tt min}/4}}{\mu_{\tt min}^2}\Big),\notag\\
    &\leq \Big(\frac{\bE[\sum_{n=1}^{n_0(B)}R_n^\pi]}{\bE[\sum_{n=1}^{n_0(B)} X_n^\pi]}- r(\mathbf{p}^*)\Big)\cdot \frac{2B}{\mu_{\tt min}} + r_{\tt max}\Big(1+\frac{8e^{-B\mu_{\tt min}/4}}{\mu_{\tt min}^2}\Big).
    \label{eqn:opt-gap-b}
\end{align}

For a causal policy $\pi$ that satisfies the constraint, we have 
\begin{align*}
     0\geq B\cdot D^{\pi}(B) &= \mathbb{E}\Big[\sum_{n=1}^{N^\pi(B)}Y_{n}^\pi\Big]-Bc \notag\\
     &\geq \mathbb{E}\Big[\sum_{n=1}^{N^\pi(B)}(Y_{n}^\pi-c X_n^\pi)\Big] \notag
\end{align*}

Then, using the fact that $|Y_n^\pi -cX_n^\pi|\leq 1$ and equation \eqref{eqn:bound_s_n}, we can bound the constraint-violation as follows:
\begin{align*}
    \bE\Big[\sum_{n=1}^{n_0(B)} (Y_n^\pi - cX_n^\pi)\Big] \leq \frac{8e^{-B\mu_{\tt min}/4}}{\mu_{\tt min}^2}.
\end{align*}

Reorganizing the terms and using the fact that $\bE[X_n^\pi]\geq \mu_{\tt min}, \forall n$, we have,
\begin{align}
    \frac{\bE[\sum_{n=1}^{n_0(B)}Y_{n}^\pi]}{\bE[\sum_{n=1}^{n_0(B)} X_n^\pi]} &\leq c+ \frac{4e^{-B\mu_{\tt min}/4}}{B\mu_{\tt min}^2}.
    \label{eqn:violation}
\end{align}

We will optimize the finite-time performance (i.e., maximize the RHS of \eqref{eqn:opt-gap-b}) over all $\pi\in\Pi$ subject to \eqref{eqn:violation} to show that $\pi^*$ has a bounded optimality gap.

Consider the following optimization problem:
\begin{maxi*}|s|
{\pi\in\Pi}{\frac{\bE[\sum_{n=1}^{n_0(B)}R_n^\pi]}{\bE[\sum_{n=1}^{n_0(B)} X_n^\pi]}}
{}{}
\addConstraint{\frac{\bE[\sum_{n=1}^{n_0(B)}Y_{n}^\pi]}{\bE[\sum_{n=1}^{n_0(B)} X_n^\pi]}\leq c+ \frac{4e^{-B\mu_{\tt min}/4}}{B\mu_{\tt min}^2}.}
\end{maxi*}

It is shown in Lemma 1 in \cite{neely2010dynamic} that there is an optimal stationary randomized policy $\pi(\mathbf{p})$ for the optimization problem above. Thus, by letting 
\begin{align*}
    {\mathbf{p}}_{\tt per}^*(B) \in \underset{\mathbf{p}\in\Delta_K}{\arg\max}~\Big\{r(\mathbf{p}):y(\mathbf{p})\leq c+\frac{4e^{-B\mu_{\tt min}/4}}{B\mu_{\tt min}^2}.\Big\},
\end{align*}

$r\big({\mathbf{p}}_{\tt per}^*(B)\big)$ is the optimum value of the optimization problem above. From \eqref{eqn:opt-gap-b}, the optimality gap grows linearly in $B$ at a rate $r\big({\mathbf{p}}_{\tt per}^*(B)\big)-r(\mathbf{p}^*)$. Using sensitivity analysis for Linear Fractional Programming \cite{bitran1976duality, bonnans2000stability} with constraint perturbation $\epsilon =\frac{4e^{-B\mu_{\tt min}/4}}{B\mu_{\tt min}^2}$, we have $r\big(\mathbf{p}^*_{\tt per}(B)\big)-r(\mathbf{p}^*)\leq O(c' \epsilon)$ for some constant $c'$. Substituting this result into \eqref{eqn:opt-gap-b}, we have:


\begin{equation}
    \bE[{\tt REW}^\pi(B)]-\bE[{\tt REW}^{\pi^*}(B)] \leq O\Big(\frac{8c'e^{-B\mu_{\tt min}/4}}{\mu_{\tt min}^3}\Big)+ r_{\tt max}\Big(1+\frac{8e^{-B\mu_{\tt min}/4}}{\mu_{\tt min}^2}\Big),
\end{equation}
which implies ${\tt REG}^{\pi^*}(B) \leq r_{\tt max} + e^{-\Omega(B)} = O(1)$.

\paragraph{(Bounding the constraint-violation)} By the upper bound in Proposition \ref{prop:srp}, under any stationary randomized policy $\pi(\mathbf{p})$, we have: 
\begin{equation*}
    \frac{1}{B}\bE\Bigg[\sum_{n=1}^{N^{\pi(\mathbf{p})}(B)}Y_{n}^{\pi(\mathbf{p})}\Bigg] \leq y(\mathbf{p})\Big(1+\frac{1}{B\mu_{\tt min}^2}\Big),
\end{equation*}
By definition, we have $y(\mathbf{p}^*) \leq c$. Substituting this into the above inequality, we obtain:
\begin{align*}
    \frac{1}{B}\bE\Bigg[\sum_{n=1}^{N^{\pi^*}(B)}Y_{n}^{\pi^*}\Bigg]\leq c +\frac{c}{B\mu_{\tt min}^2},
\end{align*}
which implies $D^{\pi^*}(B) \leq \frac{c}{B\mu_{\tt min}^2} = O(\frac{1}{B})$. 
\end{proof}

\section{Performance Bounds for Offline Lyapunov Policy}\label{app:proof_pp2}

\begin{proposition}[Drift-Minimizing Policy is Deterministic]\label{prop:offline-determ}
For any $n\geq 1$, under $\mathcal{F}_{n-1}$, let 
\begin{align}\label{eqn:lyapunov-dppr-psi}
    \Psi_n(k, Q_n) &= -V\frac{\bE[R_{n,k}]}{\bE[X_{n,k}]}+Q_{n}\frac{\bE[Y_{n,k}]}{\bE[X_{n,k}]}\\
    k_n &=\arg\min_{k\in \bK}~\Psi_n(k, Q_n) 
\end{align}
Then, $\mathbf{q}_n^{\piOff}= \I\{k=k_n\}$ is a solution to \eqref{eqn:drift-min}, i.e., $\mathbf{q}_n^{\piOff}\in\arg\min_{\mathbf{q}\in\Delta_K}~\Psi_n(Q_n^{\pi(\mathbf{q})})$. Thus, $I_n^{\piOff} = k_n$ for all $n\geq 1$ and $Q_n^{\piOff}$ is updated according to \eqref{eqn:q-length} under $\piOff$.
\end{proposition}

\begin{proof}
For $\mathbf{q}_n \in\mathcal{F}_{n-1}$, since $\bE[X_{n,k}]>0, \forall n,k$, 

\begin{align*}
    \Psi_n(Q_n^{\pi(\mathbf{q_n})}) &= \frac{\sum_{k=1}^K q_{n,k}\Big(-V\bE[R_{n,k}]+Q_{n}\cdot\bE[Y_{n,k}]\Big)}{\sum_{k=1}^K q_{n,k}\bE[X_{n,k}]}\\
    & = \frac{\sum_{k=1}^K q_{n,k}\bE[X_{n,k}] \frac{-V\bE[R_{n,k}]+Q_{n}\cdot\bE[Y_{n,k}]}{\bE[X_{n,k}]}}{\sum_{k=1}^K q_{n,k}\bE[X_{n,k}]}\\
    & \geq \frac{\min_{k\in\bK}\Big\{\frac{-V\bE[R_{n,k}]+Q_{n}\cdot\bE[Y_{n,k}]}{\bE[X_{n,k}]}\Big\}\cdot \sum_{k=1}^K q_{n,k}\bE[X_{n,k}] }{\sum_{k=1}^K q_{n,k}\bE[X_{n,k}]}\\
    & \geq \min_{k\in\bK}\Big\{\frac{(-V\bE[R_{n,k}]+Q_{n}\cdot\bE[Y_{n,k}])}{\bE[X_{n,k}]}\Big\}
\end{align*}
Therefore, $q_{n,k}^{\piOff} = \I\{k=k_n\}$ is a solution to \eqref{eqn:drift-min}.
\end{proof}

\begin{lemma}[First-order drift bounds for $Q_{n}$]\label{lemma:drift-bound-offline} Under Assumption \ref{assn:slater}, let there be an $\epsilon$-Slater arm for with $\epsilon > 0$. 
Let 
\begin{align*}
    l^* = \frac{Vr_{\tt max}}{\epsilon}
\end{align*}
Then, under $\piOff$, we have the following bound:
\begin{equation}
    \bE\Bigg[\Big( Q_{n+1}^\piOff-Q_{n}^\piOff \Big)\I\{Q_{n}^\piOff \geq l^*\} \Big|\mathcal{F}_{n-1}^\piOff\Bigg] \leq -\epsilon.
\end{equation}
\end{lemma}

\begin{proof}
We will skip the superscript $\piOff$ for $Q_n$ in the proof. For $\tt LyOff$ policy, the first order statistics for all random variables are assumed to be known. 

Assume arm $k$ satisfies $\epsilon$-Slater condition with $\epsilon>0$, i.e. $\bE[Y_{n,k}-cX_{n,k}]\leq -\epsilon$. Under $\mathcal{F}_{n-1}$, $Q_{n}$ is given. We will show that if arm $k'$ has $\bE[Y_{n,k'}-cX_{n,k'}] \geq 0$, then $\Psi_n(k',Q_n) - \Psi_n(k,Q_n)\geq 0$ for $Q_n\geq l^*$.
    
\begin{align*}
    \Psi_n(k',Q_n) - \Psi_n(k,Q_n) &= -V\Big(\frac{\bE[R_{n,k'}]}{\bE[X_{n,k'}]}-\frac{\bE[R_{n,k}]}{\bE[X_{n,k}]}\Big)+Q_{n} \Big(\frac{\bE[Y_{n,k'}]}{\bE[X_{n,k'}]}- \frac{\bE[Y_{n,k}]}{\bE[X_{n,k}]} \Big)
\end{align*}

\begin{align*}
    V\Big(\frac{\bE[R_{n,k'}]}{\bE[X_{n,k'}]}-\frac{\bE[R_{n,k}]}{\bE[X_{n,k}]}\Big) \leq V r_{\tt max}
\end{align*}

\begin{align*}
  Q_{n}\Big(\frac{\bE[Y_{n,k'}]}{\bE[X_{n,k'}]}- \frac{\bE[Y_{n,k}]}{\bE[X_{n,k}]} \Big) \geq Q_{n} \Big(c-c+\frac{\epsilon}{\bE[X_{n,k}]}\Big)
  \geq  Q_{n} \epsilon
\end{align*}

Therefore, when $Q_{n} \geq  \frac{V r_{\tt max}}{\epsilon}$,
\begin{align}
    \Psi_n(k',Q_n) - \Psi_n(k,Q_n) \geq \frac{V r_{\tt max}}{\epsilon}\cdot  \epsilon - V r_{\tt max}
    \geq 0
\end{align}

From Proposition~\ref{prop:offline-determ}, the drift-minimizing policy is deterministic in the offline setting. So $k_n = \arg\min_{k\in\bK} \Psi_n(k,Q_n)$ will satisfy the $\epsilon-$Slater condition. Then 
\begin{align}
    \bE\big[\big( Q_{n+1}-Q_{n} \big)\I\{Q_{n} \geq l^*\} \big|\mathcal{F}_{n-1}\big]  &= \bE[Y_{n,k_n}-cX_{n,k_n}] \leq -\epsilon
\end{align}
\end{proof}

\begin{lemma}[Lyapunov Drift]\label{lemma:lyapunov_drift}
For a causal policy $\pi$, define the Lyapunov function $L(q) = \frac{1}{2} q^2$ and the Lyapunov drift
\begin{align*}
    \Delta(Q_n^\pi) = \bE[L(Q_{n+1}^\pi)-L(Q_n^\pi)|\mathcal{F}_{n-1}^\pi].
\end{align*}
Then, under $\mathcal{F}_{n-1}^\pi$, 
\begin{align*}
     \Delta(Q_n^\pi) \leq V\bE[R_n^\pi|\mathcal{F}_{n-1}^\pi]+ \frac{\sigma^2}{2}+\frac{\delta^2}{2}+\delta+\bE\big[X_n^\pi\big|\mathcal{F}_{n-1}^\pi\big]\big(\Psi_n(Q_n^\pi) - cQ_n^\pi + \delta Q_n^\pi \big) 
\end{align*}
where $\Psi_n(Q_n^\pi)$ is defined in Definition~\ref{def:dppr}.
\end{lemma}

\begin{proof}
By definition, $Q_{n+1}^\pi = \max\big\{0, Q_{n}^\pi + Y_{n}^\pi-(c-\delta)X_n^\pi\big\}$. Under $\mathcal{F}_{n-1}^\pi$, $Q_n^\pi$ is given.
\begin{align*}
    (Q_{n+1}^\pi)^2 &\leq (Q_{n}^\pi + Y_{n}^\pi-(c-\delta)X_n^\pi)^2\\
    &=(Q_n^\pi)^2+(Y_{n}^\pi-(c-\delta)X_n^\pi)^2+2Q_n^\pi(Y_{n}^\pi-(c-\delta)X_n^\pi)
\end{align*}

Therefore, 
\begin{align*}
    \Delta(Q_n^\pi) &= \bE[L(Q_{n+1}^\pi)-L(Q_n^\pi)|\mathcal{F}_{n-1}^\pi]\\
    &\leq \frac{1}{2}\big(\bE[(Y_{n}^\pi-(c-\delta)X_n^\pi)^2|\mathcal{F}_{n-1}^\pi] + 2 \bE[Q_n^\pi(Y_{n}^\pi-(c-\delta)X_n^\pi)|\mathcal{F}_{n-1}^\pi]\big)\\
    &\leq \frac{\sigma^2}{2}+\frac{\delta^2}{2}+\delta+\bE[Q_n^\pi(Y_{n}^\pi-(c-\delta)X_n^\pi)|\mathcal{F}_{n-1}^\pi]
\end{align*}

Subtracting  $- V\bE[R_n^\pi|\mathcal{F}_{n-1}^\pi]$ on both sides,
\begin{align*}
    &\Delta(Q_n^\pi) - V\bE[R_n^\pi|\mathcal{F}_{n-1}^\pi]\\ \leq &\frac{\sigma^2}{2}+\frac{\delta^2}{2}+\delta+\bE[-V R_n^\pi+Q_n^\pi Y_{n}^\pi|\mathcal{F}_{n-1}^\pi]
   - \bE[Q_n^\pi(cX_n^\pi)|\mathcal{F}_{n-1}]+\bE[Q_n^\pi(\delta X_n^\pi)|\mathcal{F}_{n-1}^\pi]\\
   = &\frac{\sigma^2}{2}+\frac{\delta^2}{2}+\delta + \bE\big[X_n^\pi\big|\mathcal{F}_{n-1}^\pi\big]\Big( \frac{-V\bE\big[R_n^\pi|\mathcal{F}_{n-1}^\pi]+Q_n^\pi\cdot\bE[Y_{n}^\pi\big|\mathcal{F}_{n-1}^\pi\big]}{\bE\big[X_n^\pi\big|\mathcal{F}_{n-1}^\pi\big]} - cQ_n^\pi + \delta Q_n^\pi \Big)\\
   =& \frac{\sigma^2}{2}+\frac{\delta^2}{2}+\delta+\bE\big[X_n^\pi\big|\mathcal{F}_{n-1}^\pi\big]\big(\Psi_n(Q_n^\pi) - cQ_n^\pi + \delta Q_n^\pi \big) 
\end{align*}
Moving  $V\bE[R_n^\pi|\mathcal{F}_{n-1}^\pi]$ to the right hand side, we get the results.
\end{proof}

\subsection{Proof of Proposition \ref{prop:offline-opt}} \label{app:offline}

\begin{proof}
For notation simplicity, in the proof, we will use $\pi$ to represent $\piOff$ and use $\pi^*$ and  $\mathbf{p}^*=(p_{1}^*,p_{2}^*,\ldots,p_{K}^*)$ to represent the optimal stationary randomized policy (Definition \ref{def:osrp}). 

Under $\piOff$, recall that $ \Psi_n(k, Q_n^\pi) = -V\frac{\bE[R_{n,k}]}{\bE[X_{n,k}]}+Q_{n}^\pi\frac{\bE[Y_{n,k}]}{\bE[X_{n,k}]}$. At each step $n$, $k_n = \arg\min_{k\in\bK} \Psi_n(k, Q_n^\pi)$. From Lemma~\ref{lemma:lyapunov_drift}, 
\begin{align*}
    &\Delta(Q_n^\pi)-V\bE[R_n^\pi|\mathcal{F}_{n-1}^\pi]\\
    \leq & \frac{\sigma^2}{2}+\frac{\delta^2}{2}+\delta+\bE\big[X_n^\pi\big|\mathcal{F}_{n-1}^\pi\big]\big(\Psi_n(Q_n^\pi) - cQ_n^\pi + \delta Q_n^\pi \big)  \\
    =&  \frac{\sigma^2}{2}+\frac{\delta^2}{2}+\delta+\bE\big[X_n^\pi\big|\mathcal{F}_{n-1}^\pi\big]\big(\Psi_n(k_n, Q_n^\pi) - cQ_n^\pi + \delta Q_n^\pi \big)  \\
    \stackrel{(a)}{\leq} & \frac{\sigma^2}{2}+\frac{\delta^2}{2}+\delta+\bE\big[X_n^\pi\big|\mathcal{F}_{n-1}^\pi\big]\big(\Psi_n(Q_n^{\pi(\mathbf{p}^*)}) - cQ_n^\pi + \delta Q_n^\pi \big)  \\
    =&  \frac{\sigma^2}{2} + \frac{\delta^2}{2}+\delta+\bE\big[X_n^\pi\big|\mathcal{F}_{n-1}^\pi\big]\Big(\frac{-V\sum_{k=1}^K p_{k}^*\bE[R_{n,k}]}{\sum_{k=1}^K p_{k}^*\bE[X_{n,k}]} +\frac{Q_{n}^\pi\sum_{k=1}^K p_{k}^*\bE[Y_{n,k}]}{\sum_{k=1}^K p_{k}^*\bE[X_{n,k}]}- cQ_n^\pi +\delta Q_n^\pi \Big)\\
    \stackrel{(b)}{\leq}&  \frac{\sigma^2}{2} + \frac{\delta^2}{2}+\delta + \bE\big[X_n^\pi\big|\mathcal{F}_{n-1}^\pi\big]\big(-V r(p^*)+cQ_n^\pi - cQ_n^\pi +\delta Q_n^\pi \big)\\
    =& \frac{\sigma^2}{2} + \frac{\delta^2}{2}+\delta + \bE\big[X_n^\pi\big|\mathcal{F}_{n-1}^\pi\big]\big(-V r(p^*)+\delta Q_n^\pi \big)
\end{align*}
where (a) is using Proposition~\ref{prop:offline-determ} and (b) is using the definition of $\pi^*$ (Definition~\ref{def:osrp}).

Taking expectation on both sides, since $X_n^\pi \leq 1$,
\begin{align*}
    \bE[L(Q_{n+1}^\pi)] - \bE[L(Q_n^\pi)] &\leq \frac{\sigma^2}{2} + \frac{\delta^2}{2}+\delta + V\bE[R_n^\pi]-V\bE[X_n^\pi]r(p^*) + \delta \bE[Q_n^\pi]
\end{align*}

Taking telescoping sum from $N$ to $0$, 
\begin{align*}
    \bE[L(Q_N^\pi)] - \bE[L(Q_0^\pi)] &\leq \frac{N\sigma^2}{2}+ \frac{N\delta^2}{2}+N\delta+ V\sum_{n=1}^N \bE[R_n^\pi] \\
    &\quad - V \sum_{n=1}^N \bE[X_n^\pi]r(p^*)+\delta \sum_{n=1}^N \bE[Q_n^\pi]
\end{align*}

Rearrange the terms, since $\bE[L(Q_N^\pi)] - \bE[L(Q_0^\pi)]\geq 0$
\begin{align*}
    \frac{\sum_{n=1}^N\bE[R_n^\pi]}{\sum_{n=1}^N \bE[X_n^\pi]} \geq r(p^*) - \frac{N\sigma^2+N\delta^2}{2V\sum_{n=1}^N \bE[X_n^\pi]} - \frac{N\delta}{V\sum_{n=1}^N \bE[X_n^\pi]} - \frac{\delta \sum_{n=1}^N \bE[Q_n^\pi]}{V\sum_{n=1}^N \bE[X_n^\pi]}
\end{align*}

Let $N= n_0(B)=\lceil 2B/\mu_{\tt min}\rceil$, 
\begin{align*}
    \frac{\bE[\sum_{n=1}^{n_0(B)}R_n^\pi]}{\bE[\sum_{n=1}^{n_0(B)} X_n^\pi]} &\geq r(p^*) - \frac{n_0(B)(\sigma^2+\delta^2)}{2V n_0(B) \mu_{\tt min}}-\frac{n_0(B)\delta}{V n_0(B) \mu_{\tt min}} - \frac{\delta \sum_{n=1}^{n_0(B)} \bE[Q_n^\pi]}{V n_0(B) \mu_{\tt min}}\\
    & \geq r(p^*) - \frac{\sigma^2+\delta^2+2\delta}{2V \mu_{\tt min}} - \frac{\delta \sum_{n=1}^{n_0(B)} \bE[Q_n^\pi]}{V n_0(B) \mu_{\tt min}}
\end{align*}

Using the decomposition for regret (Lemma~\ref{lemma:regret-decomp}),
\begin{align}
    \bE[{\tt REW}^{\pi^*}(B)]-\bE[{\tt REW}^{\pi}(B)] &\leq \Big(r(p^*)- \frac{\bE[\sum_{n=1}^{n_0(B)}R_n^\pi]}{\bE[\sum_{n=1}^{n_0(B)} X_n^\pi]}\Big)\cdot \frac{2B}{\mu_{\tt min}}  + \frac{r_{\tt max}}{\mu_{\tt min}^2}\Big(1+8e^{-B\mu_{\tt min}/4}\Big) \notag\\
    &\leq \frac{B(\sigma^2+\delta^2+2\delta)}{2V \mu_{\tt min}^2}+\frac{\delta \sum_{n=1}^{n_0(B)} \bE[Q_n^\pi]}{V \mu_{\tt min}}+ \frac{r_{\tt max}}{\mu_{\tt min}^2}\Big(1+8e^{-B\mu_{\tt min}/4}\Big) \label{eqn:regret_offline_V}
\end{align}

Using the definition $Q_{n+1} = \max\big\{0, Q_{n} + Y_{n}-(c-\delta)X_n\big\}$,
\begin{align*}
    Q_{n+1} \geq Q_n + Y_{n}-(c-\delta)X_n
\end{align*}

Taking telescoping sum from $N$ to $0$, 
\begin{align*}
   \sum_{n=1}^N Y_{n}-(c-\delta) \sum_{n=1}^N X_n \leq Q_N
\end{align*}

Taking expectation on both sides, under $\piOff$ 
\begin{align*}
   \sum_{n=1}^N \bE[Y_{n}^\pi]-(c-\delta) \sum_{n=1}^N \bE[X_n^\pi] \leq \bE[Q_N^\pi]
\end{align*}

Let $n=n_0(B)$, 
\begin{align*}
   \frac{\bE\Big[\sum_{n=1}^{n_0(B)}Y_n^\pi\Big]}{\bE\Big[\sum_{n=1}^{n_0(B)}X_n^\pi\Big]}-c \leq \frac{\bE[Q_{n_0(B)}^\pi]}{\sum_{n=1}^{n_0(B)}\bE[X_i^\pi]} - \delta\leq \frac{\bE[Q_{n_0(B)}^\pi]}{n_0(B)\mu_{\tt min}}-\delta
\end{align*}

Using the decomposition for constraint-violation (Lemma~\ref{lemma:constraint-decomp}),
\begin{align}
     D^{\pi}(B) &\leq  \frac{\bE[Q_{n_0(B)}]}{n_0(B)\mu_{\tt min}} \cdot \frac{2}{\mu_{\tt min}} -\frac{2\delta}{\mu_{\tt min}}+ \frac{8(y_{\tt max}-c)}{B\mu_{\tt min}^2}e^{-B\mu_{\tt min}/4} +\frac{1}{B}\notag\\
     &\leq \frac{\bE[Q_{n_0(B)}]}{B\mu_{\tt min}} -\frac{2\delta}{\mu_{\tt min}} + \frac{8(y_{\tt max}-c)}{B\mu_{\tt min}^2}e^{-B\mu_{\tt min}/4} +\frac{1}{B}\label{eqn:constraint_offline_V}
\end{align}

Let $l^* = \frac{V r_{\tt max}}{\epsilon}$, from Theorem 2.3 in \cite{hajek1982hitting} and Lemma~\ref{lemma:drift-bound-offline}, for any $n$ and $z>l^*$,  $\bP[Q_{n}^\pi > z] \leq ae^{-\eta z}$ with some $\eta >0$ . Therefore, 
\begin{align}
    \mathbb{E}[Q_{n}^\pi] &= \mathbb{E}[Q_{n}^\pi\mathbb{I}(Q_{n}^\pi<l^*)]+\int_{l^*}^\infty \mathbb{P}(Q_{n}^\pi>z) dz + l^* \mathbb{P}(X>l^*)\notag\\
    &\leq l^*+\int_{z_0}^\infty a e^{-\eta z} dz + l^* a e^{-\eta l^*\notag}\\
    & =  \frac{V r_{\tt max}}{\epsilon} +  \frac{V r_{\tt max}}{\eta \epsilon} a e^{-\eta l^*} + \frac{V r_{\tt max}}{\epsilon} a e^{-\eta l^*}\notag\\
    &= O\Big( \frac{V r_{\tt max}}{\epsilon}\Big)\label{eqn:exp_Q_n}
\end{align}

Therefore, combining \eqref{eqn:regret_offline_V}, \eqref{eqn:constraint_offline_V} and \eqref{eqn:exp_Q_n},
\begin{align}
    \bE[{\tt REW}^{\pi^*}(B)]-\bE[{\tt REW}^{\pi}(B)] = O\Bigg(\frac{\sigma^2 B}{V\mu_{\tt min}^2}+\frac{\delta r_{\tt max} B}{\epsilon \mu_{\tt min}^2}\Bigg).
\end{align}

\begin{align}
     D^{\pi}(B) = O\Big(\frac{1}{B} + \frac{V r_{\tt max}}{B\mu_{\tt min}\epsilon}-\frac{\delta}{\mu_{\tt min}}\Big).
\end{align}
The results in Proposition~\ref{prop:offline-opt} follows by the asymptotic optimality of $\pi^*$ (Proposition~\ref{prop:optimality-gap}).
\end{proof}

\section{Performance Bounds for Online Lyapunov Policy} \label{app:online_bound_proof}
\subsection{Preliminary Results}
\begin{lemma}[Bounds for $\Psi_n$ under $\piLy$]\label{lemma:bound-psi}
Let
\begin{equation} 
    A_n^\pi = \bigcap_{t = 1}^n\bigcap_{k\in\bK}\Big\{\big| \widehat{r}_{t,k}^\pi - r_k\big| \leq {\tt rad}_k(n, \alpha)\frac{1+r_k}{\bE[X_{1,k}]}\Big\}\bigcap\Big\{\big| \widehat{y}_{t,k}^\pi - y_{k}\big| \leq {\tt rad}_k(n, \alpha)\frac{1+y_{k}}{\bE[X_{1,k}]}\Big\},
\end{equation}
be the high-probability event \yz{(Lemma 2 in \cite{cayci2020group})} under policy $\piLy$. Then, given $Q_n$,
\begin{equation}
    \I_{A_{n-1}^\piLy}\cdot\Big(\Psi_n({I_n^\piLy}, Q_n)-\Psi_n(k_n,Q_n)\Big) \leq 2{\tt rad}_{I_n^\piLy}(n,\alpha)\cdot\Big(\frac{V(1+r_{\tt max})}{\mu_{\tt min}}+\frac{Q_{n}(1+y_{\tt max})}{\mu_{\tt min}}\Big),
\end{equation}
where $k_n = \arg\min_k\Psi_n(k,Q_n)$. Moreover, we have the following bound:
\begin{equation}
    \I_{\big(A_{n-1}^\piLy\big)^c}\Big(\Psi_n({I_n^\piLy}, Q_n)-\Psi_n(k_n,Q_n)\Big) \leq Vr_{\tt max}+n(1+y_{\tt max}).
\end{equation}
\end{lemma}
\begin{proof}
    For simplicity, let $$\nu_n(V, Q_n) = \Big(\frac{V(1+r_{\tt max})}{\mu_{\tt min}}+\frac{Q_n(1+y_{\tt max})}{\mu_{\tt min}}\Big).$$ Assume that $A_{n-1}^\piLy$ holds, and $I_n^\piLy = k$ such that $$\Psi_n(k,Q_n) > \Psi_n(k_n,Q_n) +  2{\tt rad}_{k}(n,\alpha)\cdot\nu_n(V, Q_n).$$ Then, given $Q_n$, we have the following:
    \begin{align}
        \begin{aligned}
        \widehat{\Psi}_n(k_n,Q_n) - {\tt rad}_{k_n}(n,\alpha)\cdot\nu_n(V, Q_n) &\leq \Psi_n(k_n,Q_n),\\
        &\leq \Psi_n(k,Q_n)-2{\tt rad}_{k}(n,\alpha)\cdot\nu_n(V, Q_n),\\
        &\leq \widehat{\Psi}_n(k, Q_n)-{\tt rad}_{k}(n,\alpha)\cdot\nu_n(V, Q_n),
        \label{eqn:pac-bound}
    \end{aligned}
    \end{align}
    since $\widehat{r}_{n,k}$ and $\widehat{y}_{n,k,m}$ (thus $\widehat{\Psi}_n(k,Q_n),\forall k\in\bK$) are concentrated around the true values in the event $A_{n-1}^\piLy$. \eqref{eqn:pac-bound} is a contradiction since we assumed $I_n^\piLy = k$ but $k_n=\arg\min_k\Psi_n(k,Q_n)$ is a more favorable choice, which concludes the proof of the first part. 
    
    For the second part, note that $Q_{n}\leq n$ for all $n$ since $Y_{n,k}\leq 1$ almost surely for all $n,k,m$. Therefore, pessimistic upper and lower bounds on $\Psi_n(k,Q_n)$ yield the result.
\end{proof}

\begin{lemma}[First-order drift bounds for $Q_{n}$]\label{lemma:drift-bound} Under Assumption \ref{assn:slater}, let there be an $\epsilon$-Slater arm for some $\epsilon > 0$. Let
\begin{align*}
    l^* = \frac{2Vr_{\tt max}}{\epsilon}+\frac{V(1+r_{\tt max})}{1+y_{\tt max}}.
\end{align*}

Then, under {\tt LyOn}, we have the following bound for some $\epsilon_0>0$:
\begin{equation}
    \bE\Bigg[\Big( Q_{n+1}^\piLy-Q_{n}^\piLy \Big)\I_{A_{n-1}^\piLy}\I\{Q_{n}^\piLy \geq l^*\} \Big|\mathcal{F}_{n-1}^\piLy\Bigg] \leq -
    \epsilon_0.
\end{equation}
\end{lemma}

\begin{proof}
Since under $\mathcal{F}_{n-1}^\piLy$, $Q_{n}^\piLy$ is given, we will skip the superscript for $Q_n$ in the proof.

From Lemma~\ref{lemma:bound-psi}, if $A_{n-1}^\piLy$ holds, 
\begin{align*}
    \Big(\Psi_n(I_n^\piLy, Q_n)-\Psi_n(k_n,Q_n)\Big) \leq 2{\tt rad}_{I_n^\piLy}(n,\alpha)\cdot\Big(\frac{V(1+r_{\tt max})}{\mu_{\tt min}}+\frac{Q_{n}(1+y_{\tt max})}{\mu_{\tt min}}\Big)
\end{align*}

Then using the definition of $\Psi_n$,  if $Q_n\geq \frac{Vr_{\tt max}}{\epsilon}$, we have 
\begin{align*}
    Q_n\frac{\bE[Y_{n,I_n^\piLy}]}{\bE[X_{n,I_n^\piLy}]}&\leq V\frac{\bE[R_{n,I_n^\piLy}]}{\bE[X_{n,I_n^\piLy}]} - V\frac{\bE[R_{n,k_n}]}{\bE[X_{n,k_n}]}+Q_n \frac{\bE[Y_{n,k_n}]}{\bE[X_{n,k_n}]}\\
    &\quad +2{\tt rad}_{I_n^\piLy}(n,\alpha)\cdot\Big(\frac{V(1+r_{\tt max})}{\mu_{\tt min}}+\frac{Q_{n}(1+y_{\tt max})}{\mu_{\tt min}}\Big)\\
    &\stackrel{(a)}{\leq} Vr_{\tt max} + Q_n c-Q_n \epsilon+2{\tt rad}_{I_n^\piLy}(n,\alpha)\cdot\Big(\frac{V(1+r_{\tt max})}{\mu_{\tt min}}+\frac{Q_{n}(1+y_{\tt max})}{\mu_{\tt min}}\Big)\\
    &\leq Vr_{\tt max} + Q_n c-Q_n \epsilon+ \Big(\frac{V(1+r_{\tt max})}{1+y_{\tt max}}+Q_{n}\Big)\frac{\epsilon}{2}
\end{align*}
where (a) is using Lemma~\ref{lemma:drift-bound-offline} that $k_n$ satisfies $\epsilon$-Slater condition when $Q_n\geq \frac{Vr_{\tt max}}{\epsilon}$; (b) is using the fact that initial exploration makes $2{\tt rad}_{I_n^\piLy}(n,\alpha) \leq \frac{\epsilon}{2}$.

To satisfy $\bE[Y_{n,I_n^\piLy}] \leq c \bE[X_{n,I_n^\piLy}]$, let the right-hand-side be less than $Q_n c$, we have
\begin{align}
    Q_n &\geq \frac{2Vr_{\tt max}}{\epsilon}+\frac{V(1+r_{\tt max})}{1+y_{\tt max}}= l^*
\end{align}

Therefore, if $A_{n-1}^\piLy$ holds, $I_n^\piLy$ will satisfy the $\epsilon$-Slater condition when $Q_n\geq l^*$,
\begin{equation}
    \bE\Big[\Big( Q_{n+1}^\piLy-Q_{n}^\piLy \Big)\I_{A_{n-1}^\piLy}\I\{Q_{n}^\piLy \geq l^*\} \Big|\mathcal{F}_{n-1}^\piLy\Big] \leq -
    \epsilon_0.
\end{equation}

\end{proof}

\begin{lemma}[Maximal inequality for $Q_n^\piLy$] \label{lemma:max-Q}
Under Assumption \ref{assn:slater}, if the problem satisfies the Slater condition with $\epsilon > 0$ and $l^*$ defined in Lemma~\ref{lemma:drift-bound}, we have the following bound under $\piLy$ for any $N>l^*$:
\begin{equation}
    \bE\Big[\max\limits_{1\leq n \leq N}Q_{n}^\piLy\big|\mathcal{F}_0\Big] = O\big(\log(N)+l^*+1\big).
\end{equation}
\end{lemma}

\begin{proof}
This part of proof is inspired by \cite{hajek1982hitting} where they bound the hitting-time with super-martingale properties satisfied at any time. We extend the setup by allowing super-martingale property only when the queue length is greater than a threshold. This requires additional effort to decompose the event as we will see below.

For notation simplicity, we will skip the superscript $\piLy$ in the proof.

From Lemma~\ref{lemma:drift-bound}, 
\begin{equation}
    \bE\Big[\Big( Q_{n+1}-Q_{n} \Big)\I_{A_{n-1}}\I\{Q_{n} \geq l^*\} \Big|\mathcal{F}_{n-1}\Big] \leq -
    \epsilon_0
\end{equation}
From concentration inequalities in \cite{cayci2019learning}, $\mathbb{P}\Big(\big(A_{n-1}\big)^c\Big) = O\big(\frac{1}{n^2}\big)$. Also $\big| Q_{n+1}-Q_{n} \big|<1$ since we assume $X_{n}$ and $Y_{n}$ are bounded in $(0,1]$.

For any $0<\eta<\lambda$, then if $\mathbb{E}\big[Q_{n+1}-Q_{n}\big|\mathcal{F}_{n-1}\big]\leq -\epsilon_0$, $\big|Q_{n+1}-Q_{n}\big|\leq 1$, we have
\begin{align}
    \mathbb{E}\big[e^{\eta (Q_{n+1}-Q_{n})}\big|\mathcal{F}_{n-1}\big] &\stackrel{(a)}{=} 1+\eta \mathbb{E}\big[Q_{n+1}-Q_{n}\big|\mathcal{F}_{n-1}\big]+\eta^2 \sum_{m=2}^\infty\frac{\eta^{m-2}}{m!}\mathbb{E}\big[\big|Q_{n+1}-Q_{n}\big|^m\big] \notag\\
    &\stackrel{(b)}{\leq} 1-\epsilon_0 \eta + \eta^2\sum_{m=2}^\infty\frac{\lambda^{m-2}}{m!} \notag\\
    &\stackrel{(c)}{=} 1-\epsilon_0 \eta + \eta^2 \frac{e^\lambda-1-\lambda}{\lambda^2}. \label{rho_def_2}
\end{align}
where (a) is using Taylor expansion on $e^{\eta (Q_{n+1}-Q_{n})}$ and linearity of expectation; (b) is using the assumption $\mathbb{E}[Q_{n+1}-Q_{n}|\mathcal{F}_{n-1}]\leq -\epsilon_0$, $|Q_{n+1}-Q_{n}|\leq 1$ and $\eta<\lambda$; (c) is using Taylor expansion on $e^\lambda$.

If $0< \eta < \min (\lambda, \frac{\epsilon_0\lambda^2}{e^\lambda-1-\lambda})$, then $\rho = 1-\epsilon_0 \eta + \eta^2 \frac{e^\lambda-1-\lambda}{\lambda^2} < 1$. 

We select $\lambda = 1$, $\rho = 1-\eta\epsilon_0+\eta^2(e-2)$, and $\eta$ satisfying
\begin{align}
\begin{cases}
    \eta = \min\big\{1,\frac{\epsilon_0}{e-2}, \eta' \big\} \label{eta_def}\\
    \eta'>0, \;\frac{1}{4}(e^{3\eta'}-e^{\eta'})<1 
\end{cases}
\end{align}

Using the definition of $\eta$ and $\rho$ above, from \eqref{rho_def_2},
\begin{align}
\mathbb{E}[e^{\eta(Q_{n+1}-Q_n)}\I\{Q_n \geq l^*\}\I_{A_{n-1}}|\mathcal{F}_{n-1}] &\leq \rho <1 \label{neg_drift_2}
\end{align}

Since $|Q_{n+1}-Q_n|\leq 1$,
\begin{align}
\mathbb{E}[e^{\eta(Q_{n+1}-Q_n)}\I\{Q_n < l^*\}|\mathcal{F}_{n-1}] &\leq e^{\eta} \label{bound_inc_2}
\end{align}

Then, conditioning on $Q_n\geq l^*$ and $A_n$,
\begin{align}
    \mathbb{E}[e^{\eta Q_{n+1}}|\mathcal{F}_0] &= \mathbb{E}[\mathbb{E}[e^{\eta Q_{n+1}}|\mathcal{F}_{n-1}]|\mathcal{F}_0] \notag\\
    &= \mathbb{E}[\mathbb{E}[e^{\eta Q_{n+1}}\I\{Q_n \geq l^*\}+e^{\eta Q_{n+1}}\I\{Q_n < l^*\}|\mathcal{F}_{n-1}]|\mathcal{F}_0] \notag\\
    &= \mathbb{E}[\mathbb{E}[e^{\eta Q_{n+1}}\I\{Q_n \geq l^*\}|\mathcal{F}_{n-1}]|\mathcal{F}_0] + \mathbb{E}[\mathbb{E}[e^{\eta Q_{n+1}}\I\{Q_n < l^*\}|\mathcal{F}_{n-1}]|\mathcal{F}_0] \notag\\
    &\stackrel{(a)}{\leq} \mathbb{E}[\mathbb{E}[e^{\eta Q_{n+1}}\I\{Q_n \geq l^*\}\I_{A_{n-1}}|\mathcal{F}_{n-1}]|\mathcal{F}_0] \\
    &\quad + \mathbb{E}[\mathbb{E}[e^{\eta Q_{n+1}}\I\{Q_n \geq l^*\}\I_{(A_{n-1})^c}|\mathcal{F}_{n-1}]|\mathcal{F}_0] + e^{\eta (l^*+1)} \notag\\
    &\stackrel{(b)}{\leq} \rho\mathbb{E}[ e^{\eta Q_n}|\mathcal{F}_0]+\mathbb{E}[\mathbb{E}[e^{\eta Q_{n+1}}\I\{Q_n \geq l^*\}\I_{(A_{n-1})^c}|\mathcal{F}_{n-1}]|\mathcal{F}_0]+ e^{\eta (l^*+1)} \notag\\
    &\stackrel{(c)}{\leq}\rho\mathbb{E}[ e^{\eta Q_n}|\mathcal{F}_0]+\mathbb{E}[e^{\eta Q_{n+1}}\I_{(A_{n-1})^c}|\mathcal{F}_0]+ e^{\eta (l^*+1)} \notag\\
    &\stackrel{(d)}{\leq}\rho\mathbb{E}[ e^{\eta Q_n}|\mathcal{F}_0]+\sqrt{\mathbb{E}[e^{2\eta Q_{n+1}}|\mathcal{F}_0]\cdot\mathbb{E}[\I^2_{(A_{n-1})^c}|\mathcal{F}_0]}+ e^{\eta (l^*+1)} \notag\\
    &\stackrel{(e)}{\leq} \rho\mathbb{E}[ e^{\eta Q_n}|\mathcal{F}_0]+ \frac{1}{n-1} \sqrt{\mathbb{E}[e^{2\eta Q_{n+1}}|\mathcal{F}_0]}+ e^{\eta (l^*+1)} \label{bound_exp_y_temp}
\end{align}
where (a) is using inequality \eqref{bound_inc_2}; (b) is using inequality \eqref{neg_drift_2}; (c) is upper bounding $\I\{Q_n \geq l^*\}$ by 1; (d) is using Cauchy-Schwarz inequality; (e) is using $\mathbb{E}[\I^2_{(A_{n-1})^c}|\mathcal{F}_0] \leq \frac{1}{(n-1)^2}$. 

Next we will bound $\mathbb{E}[e^{2\eta Q_{n+1}}|\mathcal{F}_0]$.
Let $D=e^\eta$, $Z_n = e^{\eta Q_n}$, then $\mathbb{E}[e^{2\eta Q_{n+1}}|\mathcal{F}_0] = \mathbb{E}[Z_{n+1}^2|\mathcal{F}_0]$. 

Since $|Q_{n+1}-Q_n|\leq 1$,
\begin{align}
    D^{-1} Z_n\leq Z_{n+1} \leq D Z_n \label{bound_z_n}
\end{align}

Since $Var(Z_{n+1}|\mathcal{F}_0) = \mathbb{E}[Z_{n+1}^2|\mathcal{F}_0] - (\mathbb{E}[Z_{n+1}|\mathcal{F}_0])^2$,
\begin{align*}
\frac{\mathbb{E}[Z_{n+1}^2|\mathcal{F}_0]}{(\mathbb{E}[Z_{n+1}|\mathcal{F}_0])^2} &= 1+ \frac{Var(Z_{n+1}|\mathcal{F}_0)}{(\mathbb{E}[Z_{n+1}|\mathcal{F}_0])^2}\\
&\stackrel{(a)}{\leq} 1+ \frac{\frac{1}{4}(D-D^{-1}) \mathbb{E}[Z_n^2|\mathcal{F}_0]}{D^{-2} (\mathbb{E}[Z_n|\mathcal{F}_0])^2}\\
&= 1+ \frac{\frac{1}{4}(D^3-D) \mathbb{E}[Z_n^2|\mathcal{F}_0]}{(\mathbb{E}[Z_n|\mathcal{F}_0])^2}
\end{align*}
where (a) is using the bounded condition of $Z_{n+1}$ in \eqref{bound_z_n} and upper bound of variance for bounded random variable (Popoviciu's inequality).

From \eqref{eta_def}, let $B = \frac{1}{4}(D^3-D) < 1$.

Then 
\begin{align*}
\frac{\mathbb{E}[Z_{n+1}^2|\mathcal{F}_0]}{(\mathbb{E}[Z_{n+1}|\mathcal{F}_0])^2} &\leq 1+ B\frac{ \mathbb{E}[Z_n^2|\mathcal{F}_0]}{(\mathbb{E}[Z_n|\mathcal{F}_0])^2}
\end{align*}

Taking telescoping sum on $n$,
\begin{align}
\frac{\mathbb{E}[Z_n^2|\mathcal{F}_0]}{(\mathbb{E}[Z_n|\mathcal{F}_0])^2} &\leq \frac{1-B^{n-1}}{1-B}+ B^{n}\\
&\leq \frac{1}{1-B}+ 1 \label{2nd_moment}
\end{align}

Therefore, 
\begin{align}
    \sqrt{\mathbb{E}[e^{2\eta Q_{n+1}}|\mathcal{F}_0]} & = \sqrt{\mathbb{E}[Z_{n+1}^2|\mathcal{F}_0]}\notag\\
    &\stackrel{(a)}{\leq}  D\sqrt{\mathbb{E}[Z_n^2|\mathcal{F}_0]} \notag\\
    &\stackrel{(b)}{\leq}  D\sqrt{\Big(\frac{1}{1-B}+ 1\Big)}\mathbb{E}[Z_n|\mathcal{F}_0] \notag\\
    &= D\sqrt{\Big(\frac{1}{1-B}+ 1\Big)}\mathbb{E}[e^{\eta Q_n}|\mathcal{F}_0] \label{bound_2nd_moment}
\end{align}
where (a) is using the bounded condition of $Z_{n+1}$ in \eqref{bound_z_n} and (b) is using the inequality in \eqref{2nd_moment} .

Let $D' = D\sqrt{\Big(\frac{1}{1-B}+ 1\Big)}$, from \eqref{bound_exp_y_temp} and \eqref{bound_2nd_moment}, 
\begin{align*}
    \mathbb{E}[e^{\eta Q_{n+1}}|\mathcal{F}_0] \leq \Big(\rho + \frac{D'}{n-1}\Big)\mathbb{E}[e^{\eta Q_n}|\mathcal{F}_0] + e^{\eta (l^*+1)}
\end{align*}

Taking telescoping sum over $n$,
\begin{align}
   \mathbb{E}[e^{\eta Q_n}|\mathcal{F}_0] \leq \alpha_n e^{\eta Q_{0}}+ \beta_n e^{\eta (l^*+1)}\label{bound_exp_y}
\end{align}
where $\alpha_1=\alpha_2=\rho+D', \beta_1 = \beta_2=1$. For $n>2$,
\begin{align*}
    \alpha_n &= \alpha_1\alpha_2 \Big(\rho+\frac{D'}{(n-2)}\Big)\Big(\rho+\frac{D'}{(n-3)}\Big)\ldots\Big(\rho+D'\Big)\\
    &= \alpha_1\alpha_2 \prod_{m=1}^{n-2}\Big(\rho+\frac{D'}{m}\Big)\\
    \beta_n &= 1+\Big(\rho+\frac{D'}{(n-2)}\Big)+\Big(\rho+\frac{D'}{(n-2)}\Big)\Big(\rho+\frac{D'}{(n-3)}\Big)+\ldots\\ 
    &\quad +\Big(\rho+\frac{D'}{(n-2)}\Big)\Big(\rho+\frac{D'}{(n-3)}\Big)\ldots\Big(\rho+D'\Big)\\
    &\quad +\Big(\rho+\frac{D'}{(n-2)}\Big)\Big(\rho+\frac{D'}{(n-3)}\Big)\ldots\Big(\rho+D'\Big)^2\\
    &=1+\sum_{j=1}^{n-2} \prod_{m=1}^{j}\Big(\rho+\frac{D'}{(n-m)}\Big)  +\Big(\rho+\frac{D'}{(n-2)}\Big)\Big(\rho+\frac{D'}{(n-3)}\Big)\ldots\Big(\rho+D'\Big)^2 .
\end{align*}

Let $n_0 = \min \{n: \rho+\frac{D'}{n}<1\}$. Then let $\gamma = \Big(\rho+\frac{D'}{n_0}\Big)$, $\xi = \Big(\frac{\rho+D'}{\rho+\frac{D'}{n_0}}\Big)^{n_0}$. ($0<\gamma<1$. $\xi$ is a constant.) Using the fact $\forall n>n_0$, $\frac{1}{n}<\frac{1}{n_0}$, and $\forall 1\leq n\leq n_0$, $\frac{1}{n}\leq 1$.
\begin{align}
    \alpha_n &\leq (\rho+D')^{n_0} \Big(\rho+\frac{D'}{n_0}\Big)^{n-n_0}\notag\\
    &= \xi \gamma^{n} \label{bound_alpha_n}\\
    \beta_n &\leq \sum_{m=0}^{n-n_0} \Big(\rho+\frac{D'}{n_0}\Big)^m + \Big(\rho+\frac{D'}{n_0}\Big)^{n-n_0}(\rho+D')^{n_0-1} \notag \\
    &= \frac{1-\gamma^{n-n_0+1}}{1-\gamma}+\xi \gamma^{n}\label{bound_beta_n}
\end{align}

From \eqref{bound_exp_y}, \eqref{bound_alpha_n}, and $\eqref{bound_beta_n}$,
\begin{align}
   \mathbb{E}[e^{\eta Q_{n}}|\mathcal{F}_0] \leq \xi \gamma^{n} e^{\eta Q_{0}}+ \Big(\frac{1-\gamma^{n-n_0+1}}{1-\gamma}+\xi \gamma^{n}\Big) e^{\eta (l^*+1)} \label{bound_exp_y_2}
\end{align}

Define the fist hitting time $\tau = \inf\{n: Q_n\geq z\}$. 
Define event $H^N(z)=\{\max_{1\leq n \leq N}Q_n \geq z | \mathcal{F}_0\}$ and a disjoint set of events $H_n^N(z) = \I(\tau = n)$. $H^N(z) = \cup_{n=1}^N H_n^N(z)$. 
\begin{align*}
    e^{\eta z}\mathbb{P}(H^N(z)) &= \sum_{n=1}^N e^{\eta z}\mathbb{P}(H_n^N(z))\\
    &= \sum_{n=1}^N \mathbb{E}[e^{\eta z}\I_{H_n^N(z)}] \\
    &\stackrel{(a)}{\leq} \sum_{n=1}^N \mathbb{E}[e^{\eta Q_n}\I_{H_n^N(z)}]
\end{align*}
where (a) is using the fact $Q_n\geq z$ under event $H_n^N(z)$.

Assume $z>l^*+1$, 
\begin{align*}
    \sum_{n=1}^N \mathbb{E}[e^{\eta Q_n}\I_{H_n^N(z)}] &=  \sum_{n=1}^N \mathbb{E}[\mathbb{E}[e^{\eta Q_n}\I_{H_n^N(z)}\I\{Q_{n-1}\geq l^*\}|\mathcal{F}_{n-1}]] \\
    & \quad + \sum_{n=1}^N \mathbb{E}[\mathbb{E}[e^{\eta Q_n}\I_{H_n^N(z)}\I\{Q_{n-1}< l^*\}|\mathcal{F}_{n-1}]]
\end{align*}
Since $|Q_n-Q_{n-1}|\leq 1$, $Q_n>z>l^*+1$ is impossible when $Q_{n-1}<l^*$,
\begin{align*}
   \sum_{n=1}^N \mathbb{E}[\mathbb{E}[e^{\eta Q_n}\I_{H_n^N(z)}\I\{Q_{n-1}< l^*\}|\mathcal{F}_{n-1}]] = 0
\end{align*}

Therefore,  
\begin{align*}
   e^{\eta z}\mathbb{P}(H^N(z)|\mathcal{F}_0) &\leq \sum_{n=1}^N \mathbb{E}[\mathbb{E}[e^{\eta Q_n}\I_{H_n^N(z)}\I\{Q_{n-1}\geq l^*\}|\mathcal{F}_{n-1}]|\mathcal{F}_0]\leq \sum_{n=1}^N \mathbb{E}[e^{\eta Q_{n}}|\mathcal{F}_0]\\
   &\stackrel{(a)}{\leq} \xi \frac{1-\gamma^N}{1-\gamma}e^{\eta Q_{0}}+\Big(\frac{N}{1-\gamma}-\gamma^{-n_0+1}\frac{1-\gamma^N}{1-\gamma}+\xi \frac{1-\gamma^N}{1-\gamma}\Big)e^{\eta (l^*+1)}\\
   &\stackrel{(b)}{\leq}  \frac{\xi}{1-\gamma}e^{\eta Q_{0}}+\Big(\frac{N}{1-\gamma}+ \frac{\xi}{1-\gamma}\Big)e^{\eta (l^*+1)}
\end{align*}
where (a) is using \eqref{bound_exp_y_2} and (b) is dropping the negative terms.

Therefore, 
\begin{align}
   \mathbb{P}(\max_{1\leq n \leq N}Q_n \geq z|\mathcal{F}_0) &\leq \frac{\xi}{1-\gamma} e^{-\eta (z-Q_0)}+ \Big(\frac{N}{1-\gamma}+ \frac{\xi}{1-\gamma}\Big)e^{-\eta (z-l^*-1)}\\
   &= \Big(\frac{\xi}{1-\gamma} (e^{\eta Q_0}+e^{\eta (l^*+1)})+ \frac{ e^{\eta(l^*+1)}}{1-\gamma} N \Big)e^{-\eta z}\label{exp_bound_max}
\end{align}

Let $a = \frac{\xi}{1-\gamma} (e^{\eta Q_0}+e^{\eta (l^*+1)})$, $b = \frac{ e^{\eta(l^*+1)}}{1-\gamma}$ and $z_0 = \max\{\frac{\log (a+bN)}{\eta},l^*\}$. 

Then, let $Q_{max} = \max_{1\leq n \leq N}Q_n$. 
\begin{align*}
    \mathbb{E}[Q_{max}|\mathcal{F}_0] &= \mathbb{E}[Q_{max}\mathbb{I}(Q_{max}<z_0)|\mathcal{F}_0]+\mathbb{E}[Q_{max}\mathbb{I}(Q_{max}\geq z_0)|\mathcal{F}_0]\\
    &\stackrel{(a)}{=} \mathbb{E}[Q_{max}\mathbb{I}(Q_{max}<z_0)|\mathcal{F}_0]+\int_{z_0}^N \mathbb{P}(Q_{max}>z) dz + z_0 \mathbb{P}(Q_{max}>z_0)\\
    &\stackrel{(b)}{\leq} z_0+\int_{z_0}^N (a+bN) e^{-\eta z} dz + z_0 (a+bN) e^{-\eta z_0}\\
    & = \frac{\log (a+bN)}{\eta}+ \frac{(a+bN)}{\eta}\Big(\frac{1}{(a+bN)}-e^{-\eta N}\Big)+ \frac{\log (a+bN)}{\eta}  \frac{(a+bN)}{(a+bN)}\\
    &= \frac{2\log (a+bN)+1}{\eta}+\frac{(a+bN)e^{-\eta N}}{ \eta}
\end{align*}
where (a) uses the fact that $\mathbb{E}[X\mathbb{I}(X\geq z_0)] =  \int_{z_0}^N \mathbb{P}(X>z) dz + z_0\mathbb{P}(X>z_0)$ for a random variable $X\in[0,N]$; (b) uses the 
definition of $a$ and $b$ and the inequality~\eqref{exp_bound_max}.

Therefore, 
\begin{align*}
    \mathbb{E}\Big[\max_{1\leq n \leq N}Q_n |\mathcal{F}_0\Big] &\leq \frac{2\log \Big(\frac{\xi}{1-\gamma} (e^{\eta Q_0}+e^{\eta (l^*+1)})+ \frac{ e^{\eta(l^*+1)}}{1-\gamma} N \Big)+1}{\eta} \notag\\
   & \quad +\frac{\Big(\frac{\xi}{1-\gamma} (e^{\eta Q_0}+e^{\eta (l^*+1)})+ \frac{ e^{\eta(l^*+1)}}{1-\gamma} N \Big)e^{-\eta N}}{ \eta}\\
   & = O\big(\log(N)+l^*+1\big)
\end{align*}

\end{proof}

\subsection{Proof of Theorem \ref{thm:regret}}

\begin{proof}
For notation simplicity, in the proof, we will use $\pi$ to represent $\piLy$. From Lemma~\ref{lemma:lyapunov_drift}, under $\piLy$, 
\begin{align*}
    &\bE[L(Q_{n+1}^\pi)|\mathcal{F}_{n-1}^\pi]-\bE[L(Q_n^\pi)|\mathcal{F}_{n-1}^\pi] - V\bE[R_n^\pi|\mathcal{F}_{n-1}^\pi] \\
    \leq &\frac{\sigma^2+\delta^2+2\delta}{2} + \bE\big[X_n^\pi\big|\mathcal{F}_{n-1}^\pi\big]\Big( \Psi_n(Q_n^\pi) - cQ_n^\pi+ \delta Q_n^\pi \Big) \\
     = &\frac{\sigma^2+\delta^2+2\delta}{2}+ \bE\big[X_n^\pi\big|\mathcal{F}_{n-1}^\pi\big]\Big( \Psi_n(I_n^\pi, Q_n^\pi) - cQ_n^\pi + \delta Q_n^\pi\Big) \\
    = &\frac{\sigma^2+\delta^2+2\delta}{2} + \bE\big[X_n^\pi\big|\mathcal{F}_{n-1}^\pi\big]\Big( \I_{A_{n-1}^\pi} \Psi_n(I_n^\pi, Q_n^\pi) + \I_{(A_{n-1}^\pi)^c}\Psi_n(I_n^\pi, Q_n^\pi) - cQ_n^\pi + \delta Q_n^\pi\Big) \\
    \stackrel{(a)}{\leq} & \frac{\sigma^2+\delta^2+2\delta}{2} + \bE\big[X_n^\pi\big|\mathcal{F}_{n-1}^\pi\big]\Big( \I_{A_{n-1}^\pi} \Psi_n(k_n, Q_n^\pi) + \I_{(A_{n-1}^\pi)^c}\Psi_n(k_n, Q_n^\pi)- cQ_n^\pi + \delta Q_n^\pi\\
    & \quad + 2{\tt rad}_{I_n^\pi}(n,\alpha)\cdot\Big(\frac{V(1+r_{\tt max})}{\mu_{\tt min}}+\frac{Q_{n}^\pi(1+y_{\tt max})}{\mu_{\tt min}}\Big) +\I_{(A_{n-1}^\pi)^c}(Vr_{\tt max}+n(1+y_{\tt max})) \Big) \\
    \stackrel{(b)}{\leq} &\frac{\sigma^2+\delta^2+2\delta}{2} + \bE\big[X_n^\pi\big|\mathcal{F}_{n-1}^\pi\big]\Big(\frac{\sum_{k=1}^K p_{k}^*(-V\bE[R_{n,k}]+Q_{n}^\pi\cdot\bE[Y_{n,k}])}{\sum_{k=1}^K p_{k}^*\bE[X_{n,k}]} - cQ_n^\pi + \delta Q_n^\pi\\
    &\quad+2{\tt rad}_{I_n^\pi}(n,\alpha)\cdot\Big(\frac{V(1+r_{\tt max})}{\mu_{\tt min}}+\frac{Q_{n}^\pi(1+y_{\tt max})}{\mu_{\tt min}}\Big)
     + \I_{(A_{n-1}^\pi)^c}(Vr_{\tt max}+n(1+y_{\tt max}))\Big)\\
    = & \yzeq{}\frac{\sigma^2+\delta^2+2\delta}{2} + \bE\big[X_n^\pi\big|\mathcal{F}_{n-1}^\pi\big]\Big(-V r(p^*) +2{\tt rad}_{I_n^\pi}(n,\alpha)\cdot\Big(\frac{V(1+r_{\tt max})}{\mu_{\tt min}}+\frac{Q_{n}^\pi(1+y_{\tt max})}{\mu_{\tt min}}\Big)\\
      &\quad \yzeq{}+ \I_{(A_{n-1}^\pi)^c}(Vr_{\tt max}+n(1+y_{\tt max})\big) + y(p^*)Q_n^\pi - cQ_n^\pi + \delta Q_n^\pi\Big)\\
    \stackrel{(c)}{\leq} & \frac{\sigma^2+\delta^2+2\delta}{2} + \bE\big[X_n^\pi\big|\mathcal{F}_{n-1}^\pi\big]\Big(-V r(p^*) +2{\tt rad}_{I_n^\pi}(n,\alpha)\cdot\Big(\frac{V(1+r_{\tt max})}{\mu_{\tt min}}+\frac{Q_{n}^\pi(1+y_{\tt max})}{\mu_{\tt min}}\Big)\\
     &\quad + \I_{(A_{n-1}^\pi)^c}(Vr_{\tt max}+n(1+y_{\tt max})\big) + \delta Q_n^\pi\Big)
\end{align*}
where (a) is using Lemma~\ref{lemma:bound-psi}; (b) is using Proposition~\ref{prop:offline-determ}; (c) is using is using the definition of $\pi^*$ (Definition~\ref{def:osrp}).

Taking expectation on both sides and summing from $N$ to $0$, 
\begin{align*}
    &\bE[L(Q_N^\pi)] - \bE[L(Q_0^\pi)] \\
    \leq &\frac{N(\sigma^2+\delta^2+2\delta)}{2}+V\sum_{n=1}^N\bE[R_n^\pi] - V \sum_{n=1}^N \bE[X_n^\pi]r(p^*) + \delta \sum_{n=1}^N \bE[X_n^\pi Q_n^\pi]\\
     \quad& + 2V \sum_{n=1}^N \bE[X_n^\pi {\tt rad}_{I_n^\pi}(n,\alpha)]\frac{1+r_{\tt max}}{\mu_{\tt min}} + 2 \sum_{n=1}^N \bE[X_n^\pi {\tt rad}_{I_n^\pi}(n,\alpha)Q_n^\pi]\frac{1+y_{\tt max}}{\mu_{\tt min}}\\
     \quad& +  V \sum_{n=1}^N \bE[X_n^\pi \I_{(A_{n-1}^\pi)^c}]r_{\tt max} + \sum_{n=1}^N (\bE[X_n^\pi \I_{(A_{n-1}^\pi)^c}]\cdot n (1+y_{\tt max}))
\end{align*}

Rearrange the terms, since $\bE[L(Q_n^\pi)] - \bE[L(Q_0^\pi)]\geq 0$
\begin{align}
    \frac{\sum_{n=1}^N \bE[R_n^\pi]}{\sum_{n=1}^N \bE[X_n^\pi]} &\geq r(p^*) - \frac{N(\sigma^2+\delta^2+2\delta)}{2V\sum_{n=1}^N \bE[X_n^\pi]} - \delta \frac{\sum_{n=1}^N \bE[X_n^\pi Q_n^\pi]}{V\sum_{n=1}^N \bE[X_n^\pi]}\notag \\
    & - \frac{2\sum_{n=1}^N \bE[X_n^\pi {\tt rad}_{I_n^\pi}(n,\alpha)]}{\sum_{n=1}^N \bE[X_n^\pi]}\frac{1+r_{\tt max}}{\mu_{\tt min}}- \frac{2\sum_{n=1}^N \bE[X_n^\pi {\tt rad}_{I_n^\pi}(n,\alpha)Q_n^\pi]}{V\sum_{n=1}^N \bE[X_n^\pi]}\frac{1+y_{\tt max}}{\mu_{\tt min}}\notag\\
    & - \frac{\sum_{n=1}^N \bE[X_n^\pi \I_{(A_{n-1}^\pi)^c}]}{\sum_{n=1}^N \bE[X_n^\pi]}r_{\tt max} - \frac{\sum_{n=1}^N (\bE[X_n^\pi \I_{(A_{n-1}^\pi)^c}]\cdot n)}{V\sum_{n=1}^N \bE[X_n^\pi]} (1+y_{\tt max})\notag
\end{align}

\yz{$\bE[X_n^\pi] \geq \mu_{\tt min}$ is true for any $n$, so $\sum_{n=1}^N \bE[X_n^\pi]\geq N \mu_{\tt min}$. Therefore, }
\yz{
\begin{align}
    \frac{\sum_{n=1}^N \bE[R_n^\pi]}{\sum_{n=1}^N \bE[X_n^\pi]} &\geq r(p^*) - \frac{(\sigma^2+\delta^2+2\delta)}{2V\mu_{\tt min}} - \delta \frac{\sum_{n=1}^N \bE[X_n^\pi Q_n^\pi]}{VN\mu_{\tt min}}\notag \\
    & - \frac{2\sum_{n=1}^N \bE[X_n^\pi {\tt rad}_{I_n^\pi}(n,\alpha)]}{N}\frac{1+r_{\tt max}}{\mu_{\tt min}^2}- \frac{2\sum_{n=1}^N \bE[X_n^\pi {\tt rad}_{I_n^\pi}(n,\alpha)Q_n^\pi]}{VN}\frac{1+y_{\tt max}}{\mu_{\tt min}^2}\notag\\
    & - \frac{\sum_{n=1}^N \bE[X_n^\pi \I_{(A_{n-1}^\pi)^c}]}{N\mu_{\tt min}}r_{\tt max} - \frac{\sum_{n=1}^N (\bE[X_n^\pi \I_{(A_{n-1}^\pi)^c}]\cdot n)}{VN\mu_{\tt min}} (1+y_{\tt max})\notag
\end{align}
}
\yz{For all $n\in[1,N]$, using the facts that $\bE[Q_n^\pi] \leq \bE[\max_n Q_n^\pi]$ and $X_n^\pi \in (0,1]$, we have $\sum_{n=1}^N \bE[X_n^\pi Q_n^\pi] \leq  N\cdot\bE[\max_n Q_n^\pi]$ and $\bE[X_n^\pi {\tt rad}_{I_n^\pi}(n,\alpha)Q_n^\pi] \leq N\cdot\bE[\max_n Q_n^\pi]\cdot \bE[X_n^\pi {\tt rad}_{I_n^\pi}(n,\alpha)]$. Also, by definition of the confidence radius (Definition~\ref{def:cradius}), }
\yz{
\begin{align}
    \sum_{n=1}^N \bE[X_n^\pi {\tt rad}_{I_n^\pi}(n,\alpha)] &\leq  \sum_{n=1}^N\bE\Big[\frac{\sqrt{2\alpha \log n}}{\sqrt{T_{I_n^\pi}(n)}}\Big] \leq 2\sqrt{2\alpha K N \log N} \notag
\end{align}
}
\yz{
Therefore, 
\begin{align}
    \frac{\sum_{n=1}^N \bE[R_n^\pi]}{\sum_{n=1}^N \bE[X_n^\pi]} &\geq r(p^*) - \frac{(\sigma^2+\delta^2+2\delta)}{2V\mu_{\tt min}} - \delta \frac{\bE[\max_{n} Q_n^\pi]}{V\mu_{\tt min}}\notag \\
    & - \frac{4\sqrt{2\alpha KN\log N}}{N}\frac{1+r_{\tt max}}{\mu_{\tt min}^2}- \frac{4\sqrt{2\alpha KN\log N}\bE[\max_{n} Q_n^\pi] }{VN}\frac{1+y_{\tt max}}{\mu_{\tt min}^2}\notag\\
    & - \frac{\sum_{n=1}^N \bE[X_n^\pi \I_{(A_{n-1}^\pi)^c}]}{N\mu_{\tt min}}r_{\tt max} - \frac{\sum_{n=1}^N (\bE[X_n^\pi \I_{(A_{n-1}^\pi)^c}]\cdot n)}{VN\mu_{\tt min}} (1+y_{\tt max})\notag
\end{align}
}
\yz{
Using the fact that $\bE[\I_{(A_{n-1}^\pi)^c}]\leq \frac{1}{n^2}$ and $X_n\in(0,1]$, we have $\sum_{n=1}^N \bE[X_n^\pi \I_{(A_{n-1}^\pi)^c}] \leq \pi^2/6$ and $\sum_{n=1}^N \bE[X_n^\pi \I_{(A_{n-1}^\pi)^c}n] \leq \log N$. Therefore, }
\yz{
\begin{align}
    \frac{\sum_{n=1}^N \bE[R_n^\pi]}{\sum_{n=1}^N \bE[X_n^\pi]} &\geq r(p^*) - \frac{(\sigma^2+\delta^2+2\delta)}{2V\mu_{\tt min}} - \delta \frac{\bE[\max_{n} Q_n^\pi]}{V\mu_{\tt min}}\label{eqn:regret_online_V_1}\\
    & - \frac{4\sqrt{2\alpha KN\log N}}{N}\frac{1+r_{\tt max}}{\mu_{\tt min}^2}- \frac{4\sqrt{2\alpha KN\log N}\bE[\max_{n} Q_n^\pi] }{VN}\frac{1+y_{\tt max}}{\mu_{\tt min}^2}\label{eqn:regret_online_V_2} \\
    & - \frac{\pi^2/6}{N\mu_{\tt min}}r_{\tt max} - \frac{\log N}{VN\mu_{\tt min}} (1+y_{\tt max})\label{eqn:regret_online_V_3}
\end{align}
}
Using the expression for regret (Lemma~\ref{lemma:regret-decomp}),
\begin{align*}
    &\bE[{\tt REW}^{\pi^*}(B)]-\bE[{\tt REW}^{\pi}(B)]\\
    \leq& \Big(r(p^*)- \frac{\bE[\sum_{n=1}^{n_0(B)}R_n^\pi]}{\bE[\sum_{n=1}^{n_0(B)} X_n^\pi]}\Big)\cdot \frac{2B}{\mu_{\tt min}}  + \frac{r_{\tt max}}{\mu_{\tt min}^2}\Big(1+8e^{-B\mu_{\tt min}/4}\Big)\\
    \leq& \frac{B(\sigma^2+\delta^2+2\delta) }{V \mu_{\tt min}^2} +\delta\cdot  \frac{2B\cdot \bE[\max_{n} Q_n^\pi]}{V\mu_{\tt min}^2}+ 4\sqrt{2\alpha \cdot KB\log B}\frac{1+r_{\tt max}}{\mu_{\tt min}^2}+ \frac{\log B(1+y_{\tt max})}{V\mu_{\tt min}^2}\\
   \quad &  + \frac{4\sqrt{2\alpha\cdot KB\log B} \bE[\max_{n} Q_n^\pi]}{V}\frac{1+y_{\tt max}}{\mu_{\tt min}^2}
    + \frac{\pi^2/6}{\mu_{\tt min}}r_{\tt max} + \frac{r_{\tt max}}{\mu_{\tt min}^2}\Big(1+8e^{-B\mu_{\tt min}/4}\Big)
\end{align*}

From  Lemma~\ref{lemma:max-Q}, $\bE[\max_{1\leq n \leq n_0(B)}Q_n^\pi] = O(\log(n_0(B))+l^*+1)$. Therefore, 
\begin{align*}
    &\bE[{\tt REW}^{\pi^*}(B)]-\bE[{\tt REW}^{\pi}(B)]\\
    &= O \Bigg(\frac{r_{\tt max}\sqrt{KB\log B}}{\mu_{\tt min}^2}+\Big(\frac{\sigma^2+y_{\tt max}+\delta \log B}{V\mu_{\tt min}^2}+\frac{\delta r_{\tt max}}{\epsilon \mu_{\tt min}^2}\Big)B \Bigg)
\end{align*}
The results in Theorem~\ref{thm:regret} follows by adding the maximal regret during initial exploration $O\Big(\frac{y_{\tt max}^2 K\log B}{\epsilon^2\mu_{\tt min}^2}\Big)$ and the asymptotic optimality of $\pi^*$ (Proposition~\ref{prop:optimality-gap}).

Using the definition $Q_{n+1} = \max\big\{0, Q_{n} + Y_{n}-(c-\delta)X_n\big\}$,
\begin{align*}
    Q_{n+1} \geq Q_n + Y_{n}-(c-\delta)X_n
\end{align*}

Taking telescoping sum from $N$ to $0$, 
\begin{align*}
   \sum_{n=1}^N Y_{n}-(c-\delta) \sum_{n=1}^N X_n \leq Q_N
\end{align*}

Taking expectation on both sides, under $\piLy$ 
\begin{align*}
   \sum_{n=1}^N \bE[Y_{n}^\pi]-c \sum_{n=1}^N \bE[X_n^\pi] \leq \bE[Q_N^\pi] - \delta \sum_{n=1}^N \bE[X_n^\pi]
\end{align*}

Let $n=n_0(B)$, 
\begin{align*}
   \frac{\bE\Big[\sum_{n=1}^{n_0(B)}(Y_n^\pi)\Big]}{\bE\Big[\sum_{n=1}^{n_0(B)}(X_n^\pi)\Big]}-c \leq \frac{\bE[Q_{n_0(B)}^\pi]}{\sum_{i=1}^{n_0(B)} \bE[X_i^\pi]}  -\delta\leq \frac{\bE[Q_{n_0(B)}^\pi]}{n_0(B)\mu_{\tt min}} - \delta 
\end{align*}

Using the decomposition for regret violation (Lemma~\ref{lemma:constraint-decomp}),
\begin{align*}
     D^{\pi}(B) &\leq  \frac{\bE[Q_{n_0(B)}^\pi]}{n_0(B)\mu_{\tt min}} \cdot \frac{n_0(B)}{B} - \frac{2\delta}{\mu_{\tt min}} + \frac{8(y_{\tt max}-c)}{B\mu_{\tt min}^2}e^{-B\mu_{\tt min}/4} +\frac{1}{B}\\
     & = \frac{\bE[Q_{n_0(B)}^\pi]}{B\mu_{\tt min}} - \frac{2\delta}{\mu_{\tt min}} + \frac{8(y_{\tt max}-c)}{B\mu_{\tt min}^2}e^{-B\mu_{\tt min}/4} +\frac{1}{B}
\end{align*}

From  Lemma~\ref{lemma:max-Q}, $\bE[Q_{n_0(B)}^\pi] = O(\log(n_0(B))+l^*+1)$. 
\begin{align}
    D^{\pi}(B) &= O\Bigg( \frac{\log B}{B\mu_{\tt min}^2}+\frac{ r_{\tt max}}{B\mu_{\tt min}\epsilon}-\frac{\delta}{\mu_{\tt min}} \Bigg),
\end{align}
The results in Theorem~\ref{thm:regret} follows by adding the maximal regret during initial exploration $O\Big(\frac{y_{\tt max}^2K\log B}{\epsilon^2\mu_{\tt min}^2 B} \Big)$ and the asymptotic optimality of $\pi^*$ (Proposition~\ref{prop:optimality-gap}).
\end{proof}

\section{Additional Simulation Results}\label{app:simulations}
Figure~\ref{fig:increasing_K} shows the reward rate and constraint-violation of $\tt LyOn$ as $K$ increases ($v_0=1, \delta_0=0.5)$. New arms are added with the principle that high reward rate arm also has high penalty rate. As predicted, as $K$ increases, the $\tt LyOn$ initially has larger constraint-violation and correspondingly larger reward rate, but when budget is large enough they all converge to the optimal rate for all $K$.

\begin{figure*}[htbp]
\centering
\subfloat[]{\includegraphics[width=2.2in]{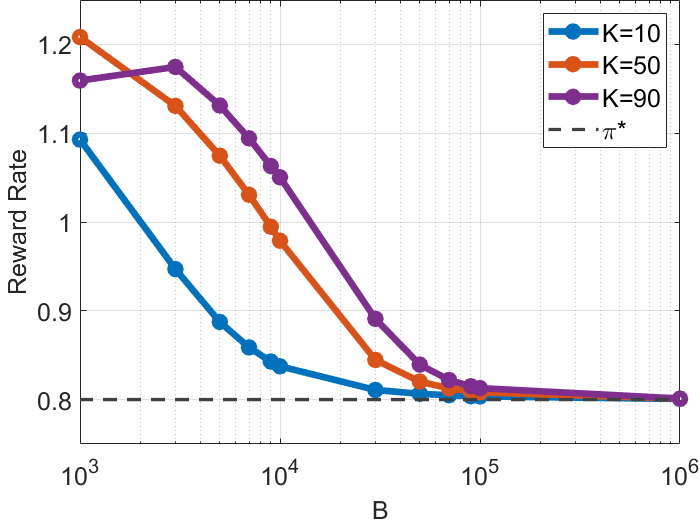}
\label{fig:K_delta_0_5_large_B_reward_rate}}
\hfil
\subfloat[]{\includegraphics[width=2.2in]{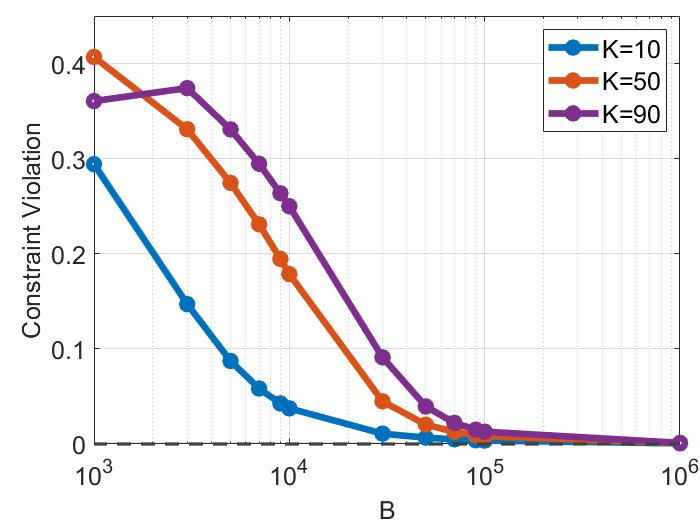}
\label{fig:K_delta_0_5_large_B_constraint_violation}}

\caption{ Constraint-violation and reward rate as $B$ grows for different number of arms. ($v_0=1, \delta_0 = 0.5$)}
\label{fig:increasing_K}
\end{figure*}

\begin{figure*}[htbp]
\centering
\subfloat[]{\includegraphics[width=2.2in]{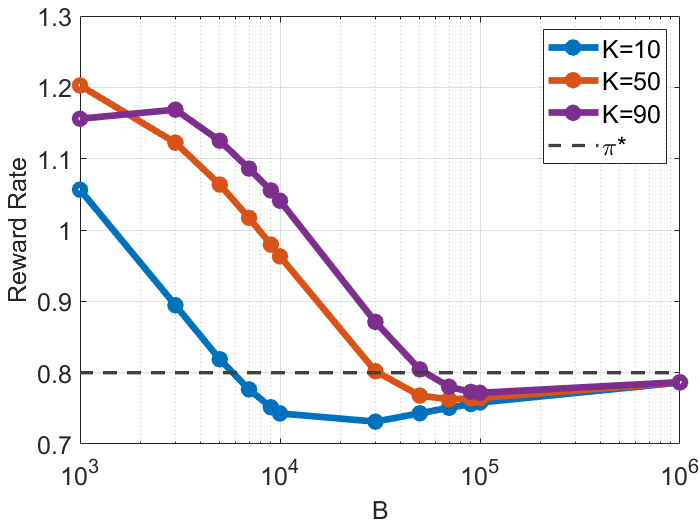}
\label{fig:K_delta_15_large_B_reward_rate}}
\hfil
\subfloat[]{\includegraphics[width=2.2in]{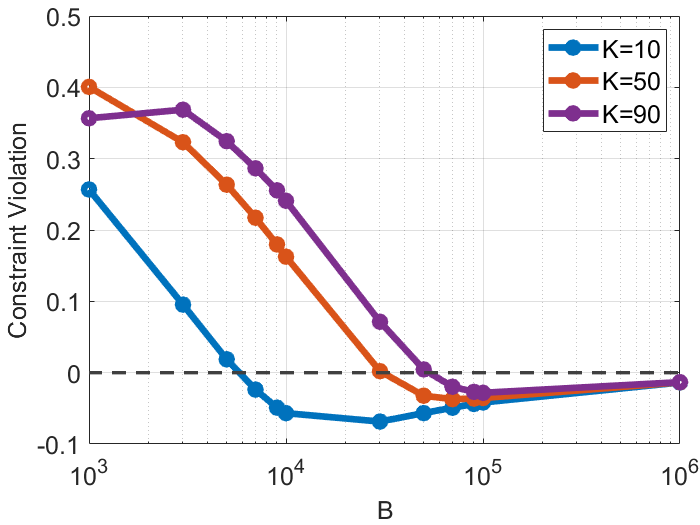}
\label{fig:K_delta_15_large_B_constraint_violation}}

\caption{ Constraint-violation and reward rate as $B$ grows for different number of arms. ($v_0=1, \delta_0 = 15$)}
\label{fig:increasing_K_large_delta}
\end{figure*}

Figure~\ref{fig:increasing_K_large_delta} further confirms the result that selecting $v_0$ and $\delta_0$ properly, we can have zero constraint-violation for different number of arms when $B$ is sufficiently large. At the same time, the reward rate still converges to the optimal.

\begin{figure*}[htbp]
\centering
\subfloat[]{\includegraphics[width=2.2in]{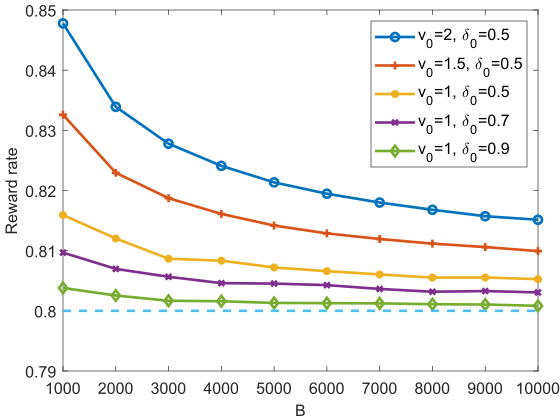}
\label{fig:Offline_V_Delta_Reward_Rate}}
\hfil
\subfloat[]{\includegraphics[width=2.2in]{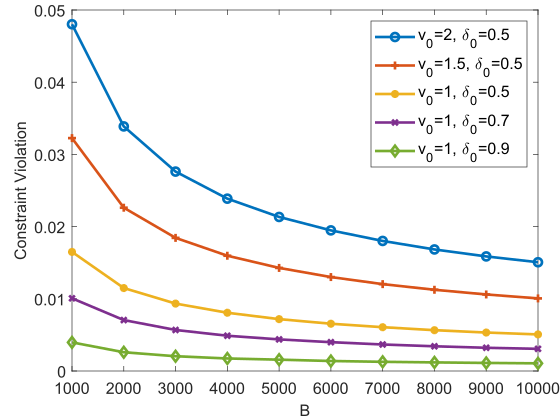}
\label{fig:Offline_V_Delta_Constraint}}
\caption{ Constraint-violation and reward rate as $B$ grows for different $v_0$ and $\delta_0$ values. (a) $\tt LyOff$ reward rate. (b) $\tt LyOff$ constraint-violation.}
\label{fig:tradeoff_delta_V_Offline}
\end{figure*}
 
\begin{figure*}[htbp]
\centering
\subfloat[]{\includegraphics[width=2.2in]{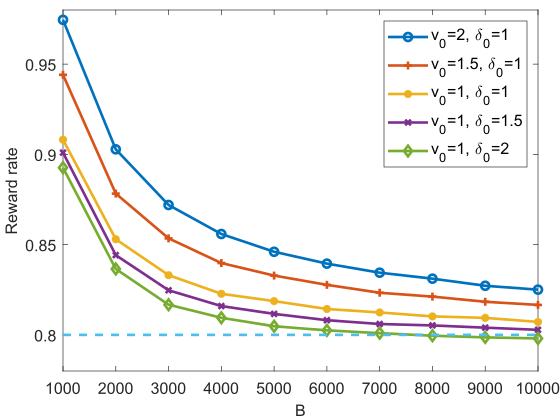}
\label{fig:Online_V_Delta_Reward_Rate}}
\hfil
\subfloat[]{\includegraphics[width=2.2in]{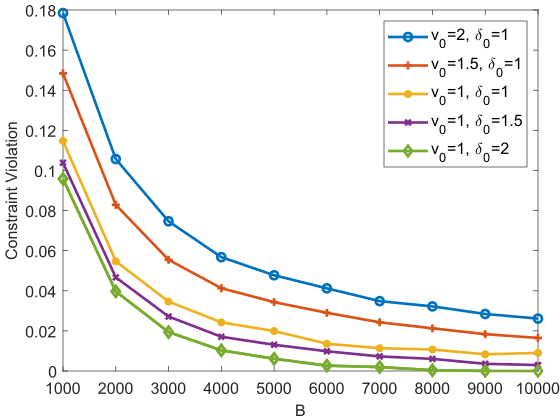}
\label{fig:Online_V_Delta_Constraint}}
\caption{ Constraint-violation and reward rate as $B$ grows for different $v_0$ and $\delta_0$ values. (a) $\tt LyOn$ reward rate. (b) $\tt LyOn$ constraint-violation.}
\label{fig:tradeoff_delta_V_Online}
\end{figure*}

Figure~\ref{fig:tradeoff_delta_V_Offline}  and Figure~\ref{fig:tradeoff_delta_V_Online} show the reward rate and constraint-violation for $\tt LyOff$ and $\tt LyOn$ algorithms with different $v_0$ and $\delta_0$ values for $K=2$ arms. Assuming $c = 0.8$, arm 1 is selected to have a high reward rate and a high penalty rate with $\bE[X_1] = 0.4, \bE[Y_1] = 0.6$, and $\bE[R_1] = 0.6$. Arm 2 is selected to have a low reward rate and a low penalty rate with $\bE[X_2] = 0.6, \bE[Y_2] = 0.3$, and $\bE[R_2] = 0.3$. The results confirm the trade-off between regret and constraint-violation when changing the values of $V$ and $\delta$. Larger $V$ or smaller $\delta$ will result in larger constraint-violation but smaller regret (higher reward rate).  On the other hand, smaller $V$ or larger $\delta$ will result in smaller constraint-violation but larger regret (smaller reward rate). 
\end{document}